\documentclass{elsarticle}

\usepackage{lineno,hyperref}
\modulolinenumbers[5]

\journal{Information Sciences}

\usepackage{multirow}
\usepackage{mathtools}
\usepackage{amsthm}
\usepackage{xfrac}
\usepackage{nicefrac}
\usepackage{todonotes}

\usepackage{algpseudocode}
\usepackage{algorithm}
\usepackage{url}
\newtheorem{theorem}{Theorem}

\usepackage[normalem]{ulem}

\begin{document}

\begin{frontmatter}

\title{Improving the performance of bagging ensembles for data streams
through mini-batching}


\author[UFSCar]{Guilherme Cassales}
\author[UoW]{Heitor Gomes}
\author[UoW]{Albert Bifet}
\author[UoW]{Bernhard Pfahringer}
\author[UFSCar]{Hermes Senger\corref{mycorrespondingauthor}}
\ead{hermes@ufscar.br}

\address[UFSCar]{Federal University of São Carlos, Brazil}
\address[UoW]{University of Waikato, New Zealand}

\cortext[mycorrespondingauthor]{Corresponding author}

\begin{abstract}

Often, machine learning applications have to cope with dynamic environments where data are collected in the form of continuous data streams with potentially infinite length and transient behavior.
Compared to traditional (batch) data mining, stream processing algorithms have additional requirements regarding computational resources and adaptability to data evolution. They must process instances incrementally because the data's continuous flow {prohibits storing} data for multiple passes.
Ensemble learning achieved remarkable predictive performance in this scenario. Implemented as a set of (several) individual classifiers, ensembles are naturally amendable for task parallelism.
However, the incremental learning and dynamic data structures used to capture the concept drift increase the cache misses and hinder the benefit of parallelism.
{This paper proposes a mini-batching strategy that can improve memory access locality and performance of several ensemble algorithms for stream mining in multi-core environments. }
With the aid of a formal framework, we demonstrate that mini-batching can significantly decrease the reuse distance {(and the number of cache misses)}.
Experiments on six different state-of-the-art ensemble algorithms applying four benchmark datasets with varied characteristics show speedups of up to 5X on 8-core processors.
These benefits come at the expense of a small reduction in predictive performance.


\end{abstract}


\end{frontmatter}


\section{Introduction}
Machine Learning (ML) models have become fundamental for many applications in different domains. 
The development of novel ML algorithms tailored to specific problem constraints is motivated by the ML omnipresence in several domains. 
In many applications, learning algorithms have to cope with dynamic environments that collect potentially unlimited data streams.
Formally, a data stream $S$ is a massive sequence of data elements $x_1, x_2, \dots, x_n$ that is, $S={\{x_i\}}_{i=1}^n$, which is potentially unbounded (n → $\infty$) \cite{silva2013data}.
Compared to traditional (batch) data mining, stream processing algorithms have additional requirements. For instance, {storing data for late processing is not a choice due to memory constraints with a potentially infinite data stream. Algorithms need to process incoming data instances incrementally in a single pass while operating under memory and response time constraints. Furthermore, as data streams present transient behavior, prediction models often need to be incremented to adapt to concept drift observed in data. }

Ensemble learning comprises a class of machine learning algorithms that achieve remarkable predictive performance. However, {ensembles} are computationally more expensive, thus demanding more computational resources and likely violating time constraints.
Although parallelism can be a good strategy for performance gains, many algorithms typically require collective communication with high overhead~\cite{JavaParallelML}. 
Algorithms derived from the Bootstrap Aggregating (Bagging ensemble algorithm do not depend on collective communications. They are typically composed of several homogeneous types, weak and unstable, learners trained independently of each other in parallel with no communication overhead. However, since ensemble models implement different data structures ~\cite{gomes2017survey}, they are not amendable for data parallelism. Moreover, because such ensembles operate on dynamic data structures used to model concept drift and non-stationary data behavior, task parallelism is the natural way to implement them. In this scenario, memory access patterns and cache memory performance become one major challenge for the parallel implementation in multi-core environments.

To assess our strategy, we use MOA (Stream Learning framework Massive Online Analysis)~\cite{bifet2010moa}, a widely known tool that provides multiple implementations of streaming algorithms, including a vast collection of ensemble methods. The rationale for using MOA for the evaluation is twofold. First, it allows the seamless evaluation of our strategy on many ensemble algorithms by changing the order MOA invokes the learners' training. {Second, we remove potential biases in the evaluation by using reliable implementations of the algorithms as the baseline in our experiments and implementing our solution in the same environment. Notice that either optimizing specific ensemble algorithms or specific learners composing the ensembles are out of this work's scope.}

The main contributions of this paper can be summarized as follows: 
\begin{enumerate}
    \item Demonstrate that the incremental learning and dynamic data structures used to capture the concept drift increase the cache misses and hinder the benefit of parallelism in multi-core environments;
    
    \item Propose and assess a parallel implementation strategy based on mini-batches, which can significantly improve memory access locality and the performance of independent ensembles;

    \item Analyze the trade-off between the predictive performance and computational efficiency when using mini-batches. We observe that it is possible to considerably speed up the execution while having a small impact on the predictive performance in most cases;
    
    \item With further tests and improvements, our strategy can be implemented as a meta-classifier in MOA to evenly support the execution of ensemble methods composed of independent classifiers. 
\end{enumerate}

This work elaborates on previous work \cite{HPCC} in the following aspects: 


{($i$) The present work provides a better description of the performance bottleneck that motivated our proposal;}
($ii$) we  included  a  completely  new  discussion  on  the  background  and  theoretical foundations of memory performance, including the proof of some theoretical bounds on the reuse distance, which explain the performance gains of our proposal;
{($iii$) we provide a thorough discussion on the related work to position our contribution in two aspects: parallel strategies for the implementation of machine learning methods and the usage of mini-batching for performance optimization; }
{($iv$) we expand our experimental framework with two additional      algorithms:  Streaming Random Patches (SRP) and OzaBagASHT;}
{($v$) we  included  two  new  platforms  in  our  experimental  framework:   Xeon E5-2650 and AMD 7702;}
($vi$) we provide empirical analysis on how our solution reduces the reuse distance and improves data locality;
($vii$) we improved the analysis of the impact on predictive performance with more  detailed  measures  (precision  and  recall  instead  of  accuracy  only). We also reduced the granularity of the experiments, increasing from 3 to 7 mini-batch sizes;
($viii$) we  included  a  new  evaluation  on  change  detection  for  different  sizes  of mini-batch compared to the baseline.

This paper is organized as follows.
The main related works are discussed in {Section} \ref{sec:relatedwork}. 
Section \ref{sec:ensembles} describes some state-of-the-art ensemble algorithms.
{A performance bottleneck of a parallel implementation of bagging ensembles is dissected in Section \ref{sec:parallelization}. Section \ref{sec:proposal} presents our proposal of a mini-batching strategy.}
Formal background on memory locality and how mini-batch improves access locality
are presented in Section \ref{sec:background}.
In Section \ref{sec:experiments} we experimentally evaluate the performance of our mini-batching strategy. The impact of batch size on the performance and quality of prediction is evaluated in Section \ref{sec:batch-size}.
Finally, our conclusions are presented in Section \ref{sec:conclusions}.


\section{Related work}
\label{sec:relatedwork}

{Research on the parallelization of machine learning methods dates to the early 90s when most papers focused on batch ML methods that require the whole dataset in the main memory to train the model. }
In early  data mining, the batch learning approach is applied by processing the whole training data (one or multiple times) to output the decision models. Then, the decision models can be applied to new production data. {Such procedure} is usually referred to by batch learning, batch-mode algorithms \cite{silva2013data}, static data mining, and others in the literature. From now on, we will use batch learning to refer to  non-stream learning methods.

Over the years, with the increase in computational power, the focus shifted from single classifiers
\cite{Wang2016PerformanceAO}
to ensembles. 
In the context of ensemble learners, many studies have been made on MapReduce (MR) frameworks
\cite{COMETMR, planet, efficient-rf}, which are not suitable for applications with requirements of low response times, even being capable of processing huge amounts of data with high scalability \cite{senger2016bsp}. Another investigation approach is using GPUs to process ensembles \cite{LE2019294, rfgpu, onl-rf, adanet}{. However, GPUs are better for data-parallel problems.}

Among the works that explored multi-core parallelism,
distributed or not, we can further subdivide it {into batch} 
\cite{HOSVD, smsvm, frem, Hussain-FPGA, Tree-efficient, adanet}
or data stream \cite{two-stage, diego-bd, autotuning, grid-ensemble} methods.
{Many works with  various ensemble methods used the Message Passing Interface (MPI) standard, such as for ensembles of improved and faster Support Vector Machine (SVM)~\cite{smsvm},
bagging decision rule ensembles~\cite{two-stage} and regression ensembles~\cite{autotuning}. 
The remaining literature presents a miscellaneous of tools and different scopes.}
In~\cite{HOSVD}, multi-classifiers are implemented using OpenMP.
In~\cite{frem}, an ensemble of SVMs is implemented using joblib (a Python multiprocessing lib) and scikit-learn.
In~\cite{adanet}, TensorFlow is used to build a scalable and extensible framework for ensembles parallelization.
In~\cite{Tree-efficient}, an efficient Random Forest (RF) implementation that improves memory access due to better data representation on machines that combine both shared and distributed memory is proposed and implemented using FREERIDE (previous work from the authors).
In~\cite{grid-ensemble}, an ensemble of J48 is parallelized for grid platforms using Java.
In~\cite{diego-bd}, a low-latency Hoeffding Tree (HT) is implemented in C++ and used in RFs.
{In general, the related works mentioned so far differ from the present work in two main aspects: focusing on the implementation and performance aspects of specific ensemble methods or batch approaches (i.e., they do not focus on stream processing). }

In more closely related work, Horak et al.~\cite{two-stage} study the impact of concurrency on memory access pattern and performance of ensembles. They proposed a two-stage bagging architecture that combines single-class recognizers with two-class discriminators to improve accuracy and allow parallel processing. They also addressed load balancing for the parallel classifier construction and used the algorithm SCALLOP as the base for the experiments to validate the architecture implemented in MPI. 
{Martinovic et al.~\cite{autotuning} enhanced a dynamic auto-tuning framework in a distributed fashion by using two strategies. They use a scalable infrastructure capable of leveraging the parallelism of the underlying platform for ensemble models to speed up the predictive capabilities while iteratively gathering production data. Results show that the approach implemented in MPI can learn the application knowledge by exploring a small fraction of the design space. }
{Qian et al.~\cite{grid-ensemble} propose a novel ensemble for data stream classification. This solution maps different raw data to multiple grids, where the first-order geometric center is used to represent and classify data. This method performs data compression, which increases the accuracy and computational efficiency. It was implemented in Java and tested in both multi-core and grid environments.}

Marrón et al.~\cite{diego-bd} propose an implementation of RF-based on vector SIMD instructions and changes the representation of Binary Hoeffding Trees (HT) to fit into the L1 cache. It was implemented in C++ and benchmarked against MOA and StreamDM using two real and eleven synthetic datasets. It is noteworthy that the authors compare the performance of a single tree and the ensemble using different hardware architectures.
Actually, these last four works are more closely related to the present work as they approach the performance of ensembles in the context of data streams. However, they are different from the present work in the following aspects. 
The works in~\cite{two-stage,diego-bd} focus on specific algorithms, SCALLOP and Binary HT, respectively. The work in~\cite{autotuning} leverages parallel processing to improve the {algorithm's parameters}, while~\cite{grid-ensemble} is focused on data compression.

The summary of related work regarding the parallelization of Machine Learning methods is shown in Table \ref{tab:related-parallel}.

\begin{table}[ht]

\advance\leftskip-1,5cm
\begin{footnotesize}
\label{tab:related-parallel}
\caption{Summary of related works in parallelized machine learning methods. }
\begin{tabular}{r|l|l|l|l}
Reference & Tool & Method & Algorithm & Platform \\ \hline
Basilico et al.~\cite{COMETMR} & MR & Batch & Ensemble RF & Hadoop \\
Ben-Haim et al.~\cite{ben-haim10a} & MPI & Stream & Single model & Distributed \\
Cyganek et al.~\cite{HOSVD} & TensorFlow OpenMP & Batch & Multi-classifier (1-13) ensemble  & Multi-core \\
Hajewski et al.~\cite{smsvm} & MPI & Batch & Ensemble SmothSVM & Distributed \\
Horak et al.~\cite{two-stage} & MPI & Stream & Bagging of SCALLOP & Multi-core \\
Hoyos-Idrobo et al.~\cite{frem} & Sci-kit learn joblib & Batch & Ensemble SVM & Multi-core \\
Hussain et al.~\cite{Hussain-FPGA} & FPGA & Batch & Single model & FPGA, Multi-core \\
Jin et al.~\cite{Tree-efficient} & - & Batch & RF & Distributed \\
Liao et al.~\cite{rfgpu} & PyCUDA Parakeet & Batch & RF & GPU \\
Marrón et al.~\cite{diego-bd} & MPI CPP & Stream & RF of binary trees & Multi-core \\
Martinovic et al.~\cite{autotuning} & MPI & Stream & Ensembles Regression & Distributed \\
Panda et al.~\cite{planet} & MR & Batch & RF & Hadoop \\
Qian et al.~\cite{grid-ensemble} & Java & Stream & Ensemble J48 & Distributed \\
Saffari et al.~\cite{onl-rf} & - & Stream & RF & GPU \\
Wang et al.~\cite{Wang2016PerformanceAO} & - & Batch & Ensemble KNN & Multi-core \\
Weill et al.~\cite{adanet} & TensorFlow & Batch & Ensemble TensorFlow & Distributed \& GPU \\
Xavier et al.~\cite{efficient-rf}& - & Batch & RF & Spark \\

\end{tabular}
\end{footnotesize}
\end{table} 

\subsection{Other mini-batching approaches}
So far, we have discussed related works that present similar motivations to the present work, i.e., optimizing machine learning algorithms and ensembles that focus on processing streaming data. Next, we discuss other related studies that optimize the performance of machine learning applications using some form of mini-batching or approaches related to the strategy proposed in the present work.

{In summary, mini-batching consists of processing small chunks containing several data instances to be processed at once instead of processing a single instance at a time. }

Variations of mini-batching have been employed with different approaches. For instance, stream processing systems such as Spark and Flink group data in small batches to improve performance and fault-tolerance~\cite{carbone2015apache, zaharia2012discretized}. 
Wang et al.~\cite{wang2012energy} proposed a  scheduling strategy to find energy-optimal batching periods for real-time tasks with deadline constraints to execute on heterogeneous sensors. 
He et al. \cite{he2010comet} proposed Comet, a stream processing system that identifies the optimal sizes of batches of data items to be processed for large-scale data streams. The proposal is based on a model named {\it Batched Stream Processing} (BSP) that focuses on modeling recurring (batch) computations on incrementally bulk-appended data streams. {Despite some similarities, this work focuses on reusing input data and intermediate results to reduce recomputing and I/O redundancies that cause bandwidth waste.}
Similar techniques that group work units into small batches have been used in other application areas, such as in web search engines~\cite{Bonacic:2015, gaioso2019performance}.
In summary, these works use mini-batching for grouping processing units into larger ones that increase the utilization of resources in multi-core or distributed processing systems. 

Similar mini-batch approaches can be used in many inversion problems. For instance, Kukreja et al.~\cite{gmd-2020-325} use mini-batches to propose a new method that combines check-pointing with error-controlled lossy compression for large-scale high-performance inversion problems. 
The method reduces movement, allowing a reduction in run time as well as peak memory. 
{In general, such methods use mini-batching to balance the amount of information used and computational costs for the optimization process used for training the learners.
This approach is different from our work, which focuses on mini-batching for improving memory access locality. }

Zhang et al.~\cite{9052125} proposed two scheduling strategies to reduce both the delay and energy consumption of executing small batches of Deep Neural Networks (DNN) tasks on edge nodes such as IoT devices. {Although the strategies also apply to CPUs, the proposal focuses on executing DNN applications on GPU devices in the edge.}
Their focus is to optimize resource utilization at the edge of the network.

{Wen et al.~\cite{Wen2020BatchEnsemble} propose a method that optimizes ensembles of Artificial Neural Networks (ANNs) and whose computational and memory costs are significantly lower than typical solutions. Such cost reduction is achieved by defining each weight matrix as the Hadamard product of a shared weight among all ensemble members and a rank-one matrix per member.} Their method yields competitive accuracy and uncertainties as typical ensembles, achieving 3X speedups at test time and 3X less memory for ensembles of 4 learners. 
This work focuses only on NN ensembles to classify image datasets.

The summary of related work regarding the parallelization of Machine Learning methods is shown in Table \ref{tab:related-batch}.

\begin{table*}[ht]
\advance\leftskip-1,5cm
\begin{footnotesize}
\label{tab:related-batch}
\caption{Summary of related works that used mini-batches}
\begin{tabular}{r|l|l}
Reference & Objective & Environment \\ \hline
Bonacic et al.~\cite{Bonacic:2015} & Increase the utilization of resources & Web search engines \\
Carbone et al.~\cite{carbone2015apache} & Fault tolerance and performance & Apache Flink \\
Gaioso et al.~\cite{gaioso2019performance} & Increase the utilization of resources & Web search engines \\
He et al. \cite{he2010comet} & Reduce recomputing and IO redundancies & Large-scale data streams \\
Kukreja et al.~\cite{gmd-2020-325} & Reduce data movement & Large-scale FWI \\

Wang et al.~\cite{wang2012energy} & Energy optimization & Real-time tasks on heterogeneous sensors \\
Wen et al.~\cite{Wen2020BatchEnsemble} & Reduce data on weight matrix & ANNs for image classification \\
Zaharia et al.~\cite{zaharia2012discretized} & Fault tolerance and performance & Apache Spark \\
Zhang et al.~\cite{9052125} & Reduce delay and energy consumption & DNNs on the edge \\


\end{tabular}
\end{footnotesize}
\end{table*}

\subsection{How our work is different from others}

Although several parallel ensemble algorithms have been proposed,
methods focusing on their efficient (parallel) implementation are seldom approached in the related literature. 
In particular, studies of memory access locality for improving the performance of ensembles are rarely approached.
To date, this is the first work to propose a strategy for improving memory access locality for parallel implementation of bagging ensembles on multi-core systems. Our work employs both measurement techniques and theoretical foundations proposed in \cite{yuan2019relational} to demonstrate the benefits of mini-batching for the implementation of ensembles. 

The present work is different from previous work as we focus on a class of ensemble algorithms composed of bagging ensembles executing in the context of data streams. Furthermore, our proposal is orthogonal to any optimization and parallel implementation of a specific learning algorithm within the ensemble like the proposals found in \cite{acceleration-cascade, pensemble-fuzzy, fastforest}.  {Being orthogonal, the mini-batching approach for ensemble optimization can be combined with other parallelization/optimization strategies that focus on specific learner algorithms within the ensemble, with potential benefits for each other. This combination, however, is out of this work's scope. }
\section{Bagging ensembles}
\label{sec:ensembles}

Bagging is one of the most used ML techniques to improve the accuracy of several weak models. {Although it was proposed over 20 years ago, Bagging and its variants (e.g., Random Forest) are still used to this day as an alternative to intricate models, such as deep neural networks, that can be challenging to train and fine-tune.}
In contrast to Boosting, Bagging does not create dependency among the base models, {facilitating the parallelization  of its execution and processing incoming data online.} Besides that, Bagging variants yield higher predictive performance in the streaming setting than Boosting or other ensemble methods that impose dependencies among its base models.
This phenomenon was observed in several empirical experiments \cite{Levbag, ozabagadwin, Gomes2017, OzaBag}. {One hypothesis to explain this phenomenon is the difficulty of effectively modeling the dependencies in a streaming scenario, as noted in \cite{gomes2017survey}.} 

Next, we present a summary description of six ensemble algorithms that evolved from the original Bagging to an online setting by Oza and Russeal~\cite{OzaBag}. We will demonstrate our mini-batching strategy on these algorithms in Section \ref{sec:proposal}.  

\textbf{Online Bagging (OzaBag - OB)}~\cite{OzaBag} is an incremental adaptation of the original Bagging algorithm. The authors demonstrate how the process of bootstrapping can be adapted to an online setting using a Poisson($\lambda=1$) distribution. In essence, instead of sampling with replacement from the original training set, in Online Bagging, the Poisson($\lambda=1$) is used to assign weights to each incoming instance, such that these weights represent the number of times an instance will {be} `repeated' to simulate bootstrapping. One concern with using $\lambda=1$ is that about $37\%$ of the instances will receive weight 0, thus not being used to train. Leaving instances out of the training set is required to approximate OzaBag to the offline version of Bagging but may be detrimental to an online learning setting~\cite{gomes2017survey}. Therefore, other works~\cite{Levbag,Gomes2017} increase the number of times an instance is used for training by increasing the $\lambda$ parameter.

\textbf{OzaBag Adaptive Size Hoeffding Tree (OBagASHT)} \cite{ozabagadwin} combines the OzaBag with Adaptive-Size Hoeﬀding Trees (ASHT). The new trees have a maximum number of split nodes and some policies to prevent the tree from growing bigger than this parameter (i.e. deleting some nodes). This algorithms' objective was to improve predictive performance by enforcing the creation of different trees. Effectively, diversity is created by having different reset-speed trees in the ensemble, according to the maximum size. The intuition is that smaller trees can adapt more quickly to changes, and larger trees can provide better performance on data with little to no changes in distribution. Unfortunately, in practice, this algorithm did not outperform variants that relied on other mechanisms for adapting to changes, such as resetting learners periodically or reactively~\cite{gomes2017survey}. 

\textbf{Online Bagging ADWIN (OBADWIN)}~\cite{ozabagadwin} combines OzaBag with the ADAptive WINdow (ADWIN) change detection algorithm. 
When a change is detected, a new classifier replaces the classifier with the worst predictive performance. 
{ADWIN keeps a variable-length window of recently seen items. The property that the window has the maximal length is statistically consistent with the hypothesis that there has been no change in the average value inside the window. It implies that} at any time, the average over the existing window can be reliably taken as an estimation of the current average in the stream, except for a very small or very recent change that is still not statistically visible.

\textbf{Leveraging Bagging (LBag)}~\cite{Levbag} extends OBADWIN by increasing the $\lambda$ parameter of the Poisson distribution to $6$, effectively causing each instance to have a higher weight and be used for training more often. 
In contrast to OBADWIN, LBag maintains one ADWIN detector per model in the ensemble and independently resets the models. 
This approach leverages the predictive performance of OBADWIN by merely training each model more often (higher weight) and resetting them individually. However, since in LBAG more training is involved, LBAG requires more memory and processing time than OB and OBADWIN. 
In \cite{Levbag}, the authors also attempted to further increase the diversity of LBag by randomizing the {ensemble's output} via random output codes. However, this approach was not very successful {compared} to maintaining a deterministic combination of the models' outputs. 

\textbf{Adaptive Random Forest (ARF)} is an adaptation of the original Random Forest algorithm to the streaming setting. Random Forest can be seen as an extension of Bagging, where further diversity among the base models (decision trees) is obtained by randomly choosing a subset of features to be used for further splitting leaf nodes. 
ARF uses the incremental decision tree algorithm Hoeffding tree~\cite{HOEFFDING_TREE} and simulates resampling as in LBag, i.e., Poisson($\lambda=6$). The Adaptive part of ARF stems from change detection and recovery strategy based on detecting warnings and drifts per tree in the ensemble. After a warning is signaled, {another model is created (a `background tree') and trained without affecting the ensemble predictions.} If the warning escalates to a drift signal, the associated tree is replaced by its background tree. Notice that in the worst case, the number of tree models in ARF can be at most double the total number of trees {due to the background trees. However, as noted in \cite{Gomes2017}} the co-existence of a tree and its background tree is often short-lived. 

\textbf{Streaming Random Patches (SRP)} \cite{SRP} is an ensemble method {specially adapted to stream classification, which combines random subspaces and online bagging.} SRP is not constrained to a specific base learner as ARF since its diversity inducing mechanisms are not built-in the base learner, i.e., SRP uses global randomization while ARF uses local randomization. 
{Even though \cite{SRP} all the experiments focused on Hoeffding trees and showed that SRP could produce deeper trees, which may lead to increased diversity in the ensemble. }

All the ensembles use a Hoeffding Tree (HT)~\cite{HOEFFDING_TREE} as a base model. An HT is an incremental tree designed to cope with massive stationary data streams. Thus, it can do splits with reasonable confidence in the data distribution while having very few instances available. This is possible because of the Hoeffding bound, which quantifies the number of observations required to estimate statistics. This guarantees that the higher the number of instances, the more similar to non-incremental trees its model gets.

\section{Parallelization of bagging ensembles}
\label{sec:parallelization}

{In this Section,} we show a straightforward parallelization of bagging ensembles. The objective is to identify its main performance bottleneck, which motivated our proposal. Although all the learners that compose an ensemble may be homogeneous in type, each one has its own (and different) model. For instance, all the learners can be implemented by {an HT, but each tree may grow in a different shape (which can change over time).} For this reason, ensemble algorithms are not amendable for data parallelism in which the same instruction operates simultaneously over different data instances.
On the other hand, task parallelism can naturally be applied as the underlying classifiers in bagging ensembles execute independently from each other and without communication.

Algorithm \ref{alg:highlevel-new} describes a task-parallel-based implementation.
This version improves the performance of the current parallel implementation of the ARF algorithm~\cite{Gomes2017}, in the latest version in MOA~\cite{bifet2010moa}, by reusing the data structures and avoiding the costs of allocating new ones for every instance to be processed.

\begin{algorithm}
  \caption{{High-level} parallel algorithm}
  \label{alg:highlevel-new}
  \begin{algorithmic}[1]
    \State {\bf Input}: an ensemble $E$, $num\_threads$, a data stream $S$
    \State $P \gets Create\_service\_thread\_pool(num\_threads)$
    \State $T \gets Create\_trainers\_collection(E)$
    \For {each arriving instance $I$ in stream $S$}
    \State $E$.classify($I$)
    \For  {each trainer $T_i$ in trainers $T$} 
    \State $k \gets poisson(\lambda)$
    \State $T_i.update(I, k)$
    \EndFor
    \For {all trainers $T$} {\bf in parallel}
    \State $W\_inst \gets I * k$
    \State $Train\_on\_instance(W\_inst)$
    \EndFor
    \If {change detected}
    \State $reset\_classifier$
    \EndIf
    \EndFor
  \end{algorithmic}
\end{algorithm}

In lines 2-3, a thread pool is started, and one Trainer (runnable) is created for each {ensemble classifier}.
For each arriving data instance (lines 4-17), votes from all the classifiers are obtained (line 5). 
{Then, the \textit{Poisson} weights are computed, and trainers are updated in lines 6-9. }
Although these steps (6 -- 9) could be {parallelized}, we chose to run sequentially for  two reasons.
First, this part corresponds to only 3\% of the computational effort for the algorithms studied in this article. 
{In addition, we chose to keep the same weights used in the baseline (sequential) algorithm to use the same random numbers in the calculation of weights. 
Maintaining the  weights will be important  (in Section \ref{sec:batch-size}) to  compare} the predictive performance of the parallel and the baseline  algorithms.

On the other hand, the training phase is more expensive. {It involves updating statistics on each tree's nodes, calculating new splits, and detecting data distribution changes (for all methods except OzaBag).}
As the training phase dominates the computational cost, parallelism is implemented (in lines 10-13) by simultaneously training many classifiers.
Finally, lines 14-16 represent the global change detector, present only on OBAdwin, where the ensemble's worst classifier will be replaced with a brand new one.

We use Java {to reuse several bagging ensembles implemented in the MOA framework}, which is widely used and tested. By reusing MOA algorithms, we provide a seamless and reproducible evaluation of our proposal, {with the added benefit of being used for many studies in the area \cite{bifet2010moa}.}
{Although Java} does not focus on implementing either energy-efficient or high-performance applications, the work in \cite{minho} shows that Java is in the top 5 languages (out of 27 tested) that need less energy and time to execute the applications.

The implementation is made in Java  using the framework ExecutorService, which implements the Fork-Join abstraction for expressing parallelism. 
The framework is available since Java 7.
It has methods to track the progress of a task and manage its termination. 
This framework's main goal is to facilitate thread management {by creating a Service with a fixed thread pool size, reserving and reusing these threads}.
Once a service has been created, tasks can be invoked by passing Runnables for it.

In essence, Fork-Join is composed of three steps: fork, computation, and join. 
In the fork step, new threads are created on-demand. 
Then, in the computation step, each thread executes one or more tasks. 
Finally, in the join step, the parallel threads synchronize and finish before continuing the program's sequential region. 
This fork-compute-join process can be repeated many times during the execution of a program. 
Fork-Join implementations usually employ thread pools that support forked task management to reduce thread creation/destruction overhead. 
These pooled threads are not destroyed when the task finishes but instead release resources and become idle~\cite{JavaParallelML}.

Although environments such as OpenMP or MPI provide remarkable support for the parallel implementation of \textit{ad hoc} algorithms, this approach is out of the scope of this study because our focus is not to optimize a specific algorithm.
Instead, the objective is to assess mini-batching as {an implementation strategy for a group of streaming algorithms already implemented in MOA.} Furthermore, Java presents {a} good performance in terms of execution time and energy efficiency for our purpose, as shown in \cite{minho}. 
The hardware used for experiments is described in Table \ref{tab:specs}.
Experiments were carried out in a dedicated environment.
Execution time is measured as wall clock time, including prediction and training steps.
We present an average of three executions for each configuration.

\begin{table}[ht]
\centering

  \caption{ Hardware specifications}
  \vspace{0,15cm}
  \label{tab:specs}
  \begin{tabular}{r|ccc}
  
    Processor        & Silver 4208 & E5-2650 & AMD 7702 \\
    \hline
    Cores/socket  & 8  & 10 & 64 \\
    Clock frequency (GHz)  & 2.1 & 2.3 & 2.0 \\
    \hline
    L1 cache (core)   & 32 KB & 32 KB & 32 KB \\
    L2 cache (core)   & 1024 KB & 256 KB & 512 KB \\
    L3 cache (shared) & 11264 KB & 25600 KB & 262144 KB\\
    \hline
    Memory (GB)    & 128  & 384 & 1024 \\
    Memory channels   & 6 &  4 & 8\\
    Maximum  bandwidth & 107.3 GiB/s & 51.2 GB/s & 204.8 GB/s \\
\end{tabular}
\end{table}

The four datasets used in the experiments are open access \footnote{Available at https://github.com/hmgomes/AdaptiveRandomForest}, and a summary of their characteristics are shown in Table \ref{tab:datasets}. A short description of each dataset is provided below:
\begin{table}[htpb]
    \centering
    \caption{Summary of dataset statistics}
    \label{tab:datasets}
    \begin{tabular}{r|cccc}
         Datasets & Airlines & GMSC & Electricity & Covertype \\ \hline
         \# of instances & 540k & 150k & 45k & 581k \\
         \# of features & 7 & 10 & 8 & 54 \\
         \# of nominal features & 4 & 0 & 1 & 45 \\
         Normalized & No & No & Yes & Yes \\
    \end{tabular}

\end{table}
\begin{itemize}
    \item  {The regression dataset from Ikonomovska inspired the Airlines dataset.} The task is to predict whether a given flight will be delayed, given information on the scheduled departure. Thus, it has 2 possible classes: delayed or not delayed.
    
    \item The Electricity dataset was collected from the Australian New South Wales Electricity Market, where prices are not fixed. These prices are affected by the demand and supply of the market itself and {are} set every 5 min. The Electricity dataset tries to identify the price changes (2 possible classes: up or down) relative to a moving average of the last 24h. An important aspect of this dataset is that it exhibits temporal dependencies.
    
    \item The give me some credit (GMSC) dataset is a credit scoring dataset where the objective is to decide whether a loan should be allowed. This decision is crucial for banks since erroneous loans lead to the risk of default and unnecessary expenses on future lawsuits. The dataset contains historical data on borrowers.

    \item The forest Covertype dataset represents forest cover type for 30 x 30 m cells obtained from the US Forest Service Region 2 resource information system (RIS) data. Each class corresponds to a different cover type. The numeric attributes are all binary.
    Moreover, there are 7 imbalanced class labels.

\end{itemize}

As depicted in  Figure \ref{fig:par-spup}, most experiments present either a negative or a negligible improvement in speedup.
Although the parallel implementation can train  several classifiers in parallel, {this operation requires few calculations. At the same time, it }demands reading and writing large data structures in memory, thus producing a significant amount of cache misses, as shown in Table \ref{tab:cache-miss-parallel}. 
Actually, parallelism increases the demand for memory bandwidth as multiple cores can execute different classifiers, filling the caches with their respective data structures. Such demand creates a bottleneck in the memory system, which hinders performance.   
The results presented here are consistent {with} previous studies reported in the literature (e.g., in \cite{Gomes2017}).

In summary, it demonstrates that the parallelism {\it per se} cannot improve the performance of streaming algorithms. {To alleviate the bottleneck discussed here and improve the performance and scalability of bagging ensembles, the next Section} proposes a strategy that can increase the reuse of large data structures in the cache.

\begin{figure}[h]
    \centering
    \advance\leftskip-2cm
    \includegraphics[width=1.3\textwidth]{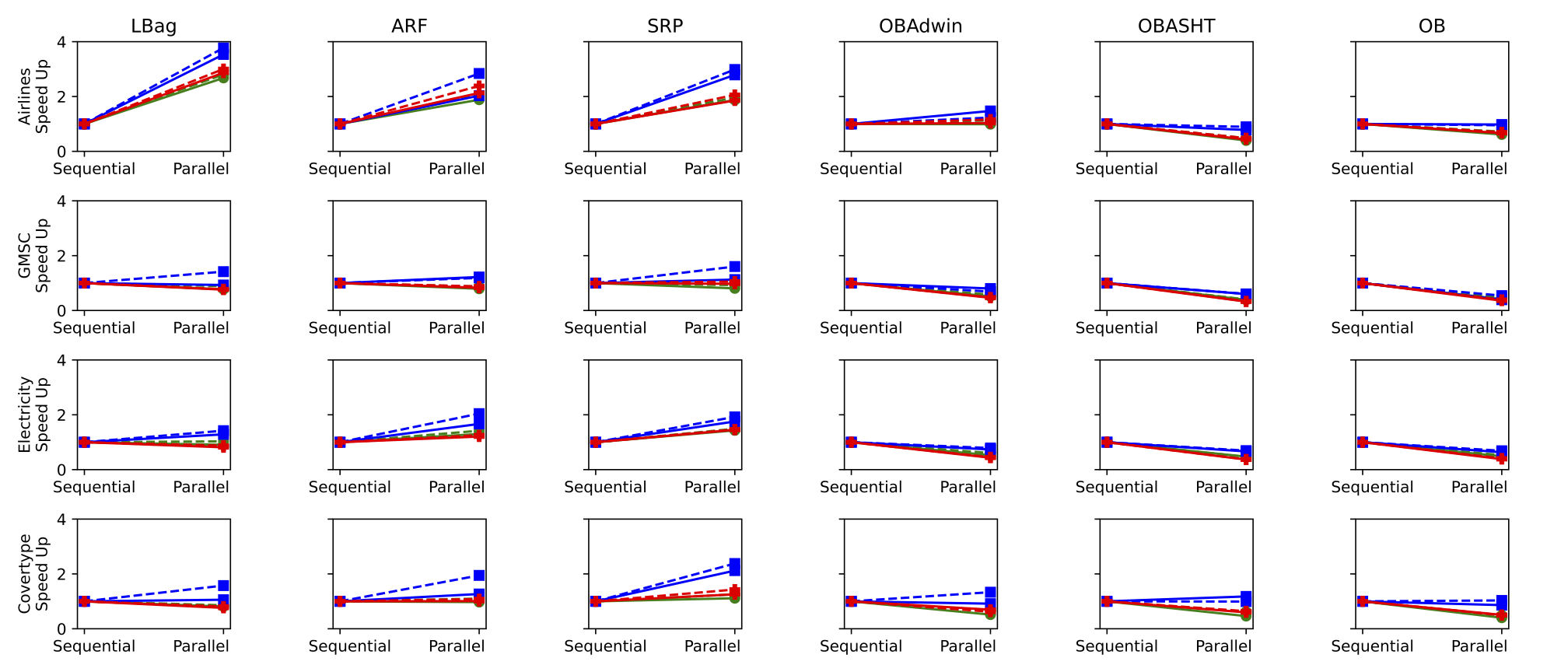}
    \caption{{\bf Speed up with 8 threads for all platforms}. Algorithms are placed on columns, while datasets are placed in different rows of the grid. {Suffix 100 and 150 indicate the size of the ensemble and are represented with solid and dashed lines, respectively.} All Y-axis are scaled to 4. Algorithm implementations: ($i$) Baseline (\textbf{Sequential}), ($ii$) Parallel  (\textbf{Parallel}).
    }
    \label{fig:par-spup}
\end{figure}


\begin{table}[ht]

\advance\leftskip-0,5cm
\begin{footnotesize}
\caption{Measures of cache {usage} for ensembles with 100 learners (Millions)}
\label{tab:cache-miss-parallel}
\begin{tabular}{c|p{1,0cm}p{1,0cm}|p{1,0cm}p{1,0cm}|p{1,0cm}p{1,0cm}|p{1,0cm}p{1,0cm}}
&  \multicolumn{2}{c}{Airlines} & \multicolumn{2}{c}{GMSC} & \multicolumn{2}{c}{Electricity} & \multicolumn{2}{c}{Covertype} \\
Algorithm & cache-miss & cache-refer & cache-miss & cache-refer & cache-miss & cache-refer & cache-miss & cache-refer \\
\hline
ARF & 40,171 & 94,910 & 2,518 & 11,366 & 882 & 4,490 & 12,652 & 65,321 \\
\hline
LBag & 45,337 & 99,010 & 2,600 & 8,962 & 508 & 2,870 & 17,809 & 104,735 \\
\hline
SRP & 45,135 & 110,900 & 5,543 & 18,487 & 2,105 & 7,520 & 65,157 & 172,089 \\
\hline
OBASHT & 4,779 & 39,986 & 531 & 4,399 & 225 & 1,714 & 5,927 & 101,370 \\
\hline
OBAdwin & 26,627 & 71,987 & 723 & 5,770 & 232 & 2,037 & 5,780 & 108,281 \\
\hline
OB & 9,423 & 27,560 & 981 & 5,580 & 221 & 1,864 & 11,314 & 94,976 \\

\end{tabular}
\end{footnotesize}
\end{table}


\section{Improving memory locality through mini-batching}
\label{sec:proposal}


Although task parallelism looks straightforward for {implementing ensembles}, bad memory usage can severely hinder their performance. 
For instance, high-frequency access to data structures larger than cache memories can raise severe performance bottlenecks. We can observe that parallelism increased the cache contention, thus increasing the number of cache misses, which explains the loss of speedup in some experiments.  


One strategy for {alleviating} this memory bottleneck is to improve the data reuse of the classifiers. {Such data reuse can be improved by processing more than one instance at a time so that the data structured loaded into the cache can be reused for processing a group of instances. }
In the present work, we use the term mini-batches to refer to a group of instances that will be processed at once by each {ensemble classifier}.
Thus, the mini-batching strategy {aims to reduce cache misses by improving cache data reuse.}

Figure \ref{fig:ensemble-figure} presents a simplified view of the ensemble working. 
The mini-batch is replicated to each learner in the ensemble, which outputs the predictions of the whole mini-batch.
{After that, there is an aggregation of the predictions} to output the final prediction of the whole ensemble.
Then, the learners use the same mini-batch to update their models. This way, the training is deferred to the end of the processing of a mini-batch.

\begin{figure}
    \centering
    \includegraphics[width=\textwidth]{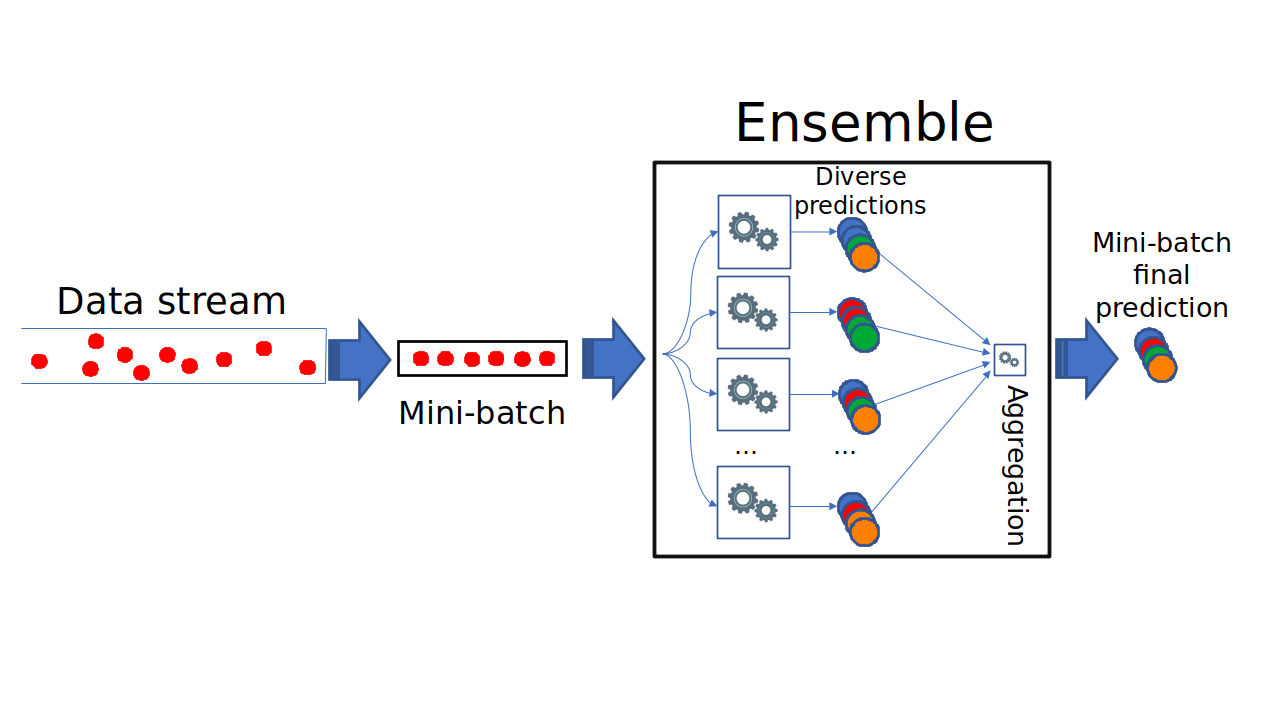}
    \caption{A simplified example of the ensemble working with mini-batch in parallel.}
    \label{fig:ensemble-figure}
\end{figure}

The introduction of mini-batching does not break the single-pass requirement for operating on data streams, as each instance will continue to be processed only once then discarded. However, the mini-batching  
defers the update of the models to the end of the mini-batch {(i.e., training starts after all the instances of the mini-batch have been classified)}.

Each learner is implemented by a task that processes training by iterating on each instance of a mini-batch instead of processing a single instance at a time. When a learner is invoked, its data structures are loaded into the upper levels of the memory hierarchy (upper-level caches) and quickly accessed to process the next instances in the same mini-batch, reducing cache misses and improving performance. 
Next, we propose mini-batching for improving the memory locality of the parallel portion of the code. 

The  Algorithm~\ref{alg:batch} shows the mini-batching strategy.
\begin{algorithm}
  \caption{mini-batching algorithm}
  \label{alg:batch}
  \begin{algorithmic}[1]
    \State {\bf Input}: an ensemble $E$, $num\_threads$, a data stream $S$, mini-batch size $L_{mb}$
    \State $P \gets Create\_service\_thread\_pool(num\_threads)$
    \State $T \gets Create\_trainers\_collection(E)$
    \For {each arriving instance $I$ in stream $S$}
    \State $B.append(I)$
    \If{$B.size() == L_{mb}$}
    \State $E$.classify($B$)
    \For  {each trainer $T_i$ in trainers $T$} 
    \State $T_i.instances.append(B)$
    \EndFor
    \For {all trainers $T$} {\bf in parallel}
    \For {each instance $I$ in $B$}
    \State $k \gets poisson(\lambda)$
    \State $W\_inst \gets I * k$
    \State $Train\_on\_instance(W\_inst)$
    \EndFor
    \If {change detected}
    \State $reset\_classifier$
    \EndIf
    \State $B.clear()$
    \EndFor
    \EndIf
    \EndFor
  \end{algorithmic}
\end{algorithm}
The first difference between the two algorithms appears in lines 5-6 of the Algorithm \ref{alg:batch}, where the ensemble will only accumulate the instances until the desired mini-batch size is met or the stream ends. When this condition is fulfilled, the algorithm performs the classification (line 7) and training (lines 8-21). In line 9, the whole mini-batch is copied to each trainer, and the weight is left to be calculated inside each trainer. Line 11 has a parallel loop that executes the trainer in parallel. Then, each trainer will iterate (sequentially) over all mini-batch instances while calculating the weight, creating the weighted instance, and training the classifier with this instance (lines 12-16). ARF and LBag, exclusively, will execute lines 17-19 as a local change detector for each classifier in the ensemble. Finally, in line 20, the instances are discarded and the buffer is flushed to begin accumulating a new mini-batch.
In OBAdwin, lines 17-19 would be outside the parallel section, {as change detection is global}.

Notice that each prediction model usually is several times larger than a single instance, and the size of the whole ensemble (almost always) overcomes the cache memory by far. Thus, it is significantly more expensive to load one model into memory than loading one instance. 
Algorithm \ref{alg:highlevel-new} loads one sample (usually small) in memory, and iterates over the ensemble classifiers (significantly larger), {thus incurring a high number of cache misses}. 
In contrast, the mini-batching fixes one ensemble model in memory, then iterates over the mini-batch samples, thus improving memory access locality by reducing the cache misses. 

Ideally, the size of ensemble models should fit into the cache memories for optimal performance. However, this is not the case for current streaming bagging ensembles  (e.g., the ones considered in this article). Actually, because the entire ensembles typically overcome the size of caches by far, the mini-batching can {significantly reduce} the number of cache misses. {Such reduction  explains} the performance gains achieved with mini-batching. 



\section{Background and preliminaries}
\label{sec:background}

Memory locality is a fundamental property and design principle for optimizing the performance of hardware, software, and algorithms \cite{yuan2019relational}. 
Locality can be defined as ``the tendency for programs to cluster references to subsets of address space for extended periods" \cite{denning2015great}.
Due to the increasing gap in processor and memory speeds \cite{jacob2010memory},
{the locality has played} a central role in optimizing the performance of operating systems over the decades.
Locality may be defined in many ways, and several metrics related to it have been proposed \cite{yuan2019relational}. 
We can use the notion of \textit{reuse distance} (RD) to evaluate how mini-batching can improve the performance and resource (i.e., cache) sharing of bagging ensembles implementations. 
We chose these definitions because they are based on direct measurements, do not depend on idealistic assumptions, and are extensions of observational stochastic \cite{yuan2019relational}.
Our goal is to reduce cache misses by improving data reuse in parallel implementations of ensembles.

From a historical perspective, memory locality has been studied over decades to optimize the memory hierarchy, operating systems, software, and algorithms design 
with recent advances in measurement techniques \cite{Ibrahim2010},
trace generation \cite{shen2008scalable},
and formal modeling   \cite{Maeda2017}. 
{In recent work} \cite{yuan2019relational}, 
Yuan et al. built upon previous works by proposing the \textit{relational theory of locality} (RTL), a theoretical framework that unifies several memory locality measures used along five decades of study and research in the field. 
RTL provides mathematical background and categorizes the measures in three different types of locality. The authors showed how such measures relate to each other and whether and how they can be inter-converted. 


Next, we discuss {the} memory locality of a stream processing system operating according to the algorithms described in the previous section. Each algorithm implements an ensemble $L$, composed by a set of learners $l_i \in E$. We refer to an individual learner as $l_i$ $(1\leq i \leq |E|)$.
A {\it stream} $S$ is a countably infinite set of data elements $s \in S$. Each stream element s:⟨v, t⟩ consists of a relational tuple $v$ conforming to some schema, with an application time value $t_{i} \in T$. We assume that the time domain $T$ is a discrete, linearly ordered, countably infinite set of time instants $t \in T$.  As the stream is potentially infinite, we assume that $T$ is bounded in the past but not necessarily in the future.
Thus, due to memory limitations and response time constraints, the algorithms need to incrementally process incoming data elements in a single pass, performing both classification and training as data elements arrive.

A \textit{trace} $N$ is a sequence of references to data or memory locations denoted by $N = mt(1,\dots,n)$, where $n$ is the trace's length.
A trace can access a set of $m$ distinct memory addresses, while the set of distinct memory addresses is be denoted by $M = {e_1,\dots,e_m}$, where $m$ is the number of distinct memory addresses in $M$. 
The model allows abstracting from any granularity issues so that a data item may be either a variable, a data block, a page, or an object.
For illustration, we can use some trace examples composed of just three data elements $a,b,c$, including those repeating them once in the same order (i.e., \textit{abc abc}), in the opposite order (i.e., \textit{abc cba}), or repeating them indefinitely (i.e., \textit{abc abc \dots}).


In essence, access locality is related to measuring the locality for each memory access.
From the five definitions of access locality provided by YUAN et al. \cite{yuan2019relational},
{we use only the definition of \textit{reuse distance sequence} or \textit{reuse distance} (RD) for short because it suffices to demonstrate that mini-batching can improve ensemble implementations' access locality.} The equivalence among the definitions is proven in \cite{yuan2019relational}.

The reuse distance (RD) is defined as ``the number of distinct data accessed since the last access to the same datum, including the reused datum”~\cite{yuan2019relational}.
The reuse distance is $\infty$ for its first access.
For a finite reuse distance, the minimum is 1 (because it includes the reused datum), and the maximum is $m$.
For example, the RD sequence is $\infty \infty \infty$ 333 for \textit{abc abc} and $\infty \infty \infty$ 135 for \textit{abc cba}. 

Before demonstrating the benefits of mini-batching, it is worth noting that stream processing ensembles have two principal operations: the classification (in line 5 of Algorithm \ref{alg:highlevel-new}) and the training (line 12 of Algorithm \ref{alg:highlevel-new}). 
The former reads a few model variables of each classifier, while the latter is (by far) the dominant operation in terms of computational cost because it performs both read and write operations to update the classifier's models.
For this reason, our analysis presented here focus on the training operation.

For the sake of illustration, Table \ref{tab:rd_cache} presents a simple example with an ensemble of $m=4$ learners processing a stream of $n=6$ data items.
Without mini-batching, the processing of the first data item produces a sequence of $m$ occurrences of $\infty$ reuse distance.
However, for finite reuse distance, the minimum reuse distance is 1 because it includes the reused datum, and the maximum is $m$.
Thus,  $\infty$ is shown only for illustration, being ignored in our analysis hereafter.
In this example, for each access $e_i$, the reuse distance {equals} the number of ensembles $m$.
{With mini-batching, each ensemble is accessed once within the mini-batch (with reuse distance $m$) and reused $b-1$ times.}
One could easily realize the benefits of mini-batching by simply substituting every $\infty$ by 1 and calculating the average reuse distance for the two executions.
Next, we demonstrate the benefits. 


\begin{table}[ht]
\advance\leftskip-2,0cm
\caption{Example: A stream of $n=6$ data items being processed by an ensemble of $m$ learners without and with mini-batching.} 
\label{tab:rd_cache}
\begin{tabular}{c p{0.1cm}p{0.1cm}p{0.1cm}p{0.1cm} p{0.1cm}p{0.1cm}p{0.1cm}p{0.1cm} p{0.1cm}p{0.1cm}p{0.1cm}p{0.1cm} p{0.1cm}p{0.1cm}p{0.1cm}p{0.1cm} p{0.1cm}p{0.1cm}p{0.1cm}p{0.1cm} p{0.1cm}p{0.1cm}p{0.1cm}p{0.1cm}}
\multicolumn{25}{c}{RD without mini-batching. A semicolon (;) denotes separation between data items.} \\
\hline 
Access sequence &$e_1$,&$e_2$,&$\dots$,&$e_m$;
&$e_1$,&$e_2$,&$\dots$,&$e_m$;
&$e_1$,&$e_2$,&$\dots$,&$e_m$;
&$e_1$,&$e_2$,&$\dots$,&$e_m$;
&$e_1$,&$e_2$,&$\dots$,&$e_m$;
&$e_1$,&$e_2$,&$\dots$,&$e_m$\\
RD sequence  &$m$,&$m$,&$\cdots$,&$m$;
&$m$,&$m$,&$\cdots$,&$m$;
&$m$,&$m$,&$\cdots$,&$m$;
&$m$,&$m$,&$\cdots$,&$m$;
&$m$,&$m$,&$\cdots$,&$m$;
&$m$,&$m$,&$\cdots$,&$m$
\\ \hline
\\
\multicolumn{25}{c}{RD with mini-batching of size $b=3$. A semicolon (;) separates  mini-batches.} \\
\hline 
Access sequence &
$e_1$,&$e_1$,&$e_1$;&
$e_2$,&$e_2$,&$e_2$;&
$\cdots \;$;&
$e_m$,&$e_m$,&$e_m$;&
$e_1$,&$e_1$,&$e_1$;&
$e_2$,&$e_2$,&$e_2$;&
$\cdots \;$;&
$e_m$,&$e_m$,&$e_m$&
\\
RD sequence  &
m,&1,&1;&
m,&1,&1;&
$\cdots$;&
m,&1,&1;&
m,&1,&1;&
m,&1,&1;&
$\cdots$;&
m,&1,&1
\\
\hline
\end{tabular}
\end{table}

{For the proofs shown in this section, we assume the amount of memory used to implement the ensemble exceeds the cache memory size (which is quite realistic). Otherwise, all accesses will hit the cache, and the order in which memory positions are accessed does not influence cache misses.}


\begin{theorem}
The reuse distance of 
an ensemble of $m$ learners
processing a $n$-length data stream  
is $\mathcal{O}(nm^2)$.  
\end{theorem}

\begin{proof}

Consider an ensemble composed of $m$ learners and a data stream composed of $n$ data elements.
Let $e_1,e_2,\dots,e_m$ be the memory locations accessed during a sequential execution of the ensemble to process.
As the model can express arbitrary granularity, {for simplicity, consider that  $e_i$ denotes one access to the data structures} of the {\it i-th} learner $l_i$ of the ensemble. 

The execution of the training operation (line 5 of Algorithm \ref{alg:highlevel-new}) will produce the access sequence $(e_1,e_2,\dots, e_m)$ for each arriving data instance because it invokes all the learners' training in this exact order. 
Thus, the training operation will produce the access sequence $(e_1,e_2,\dots,e_m)^{n}$ for the access stream.
Then, considering finite distances (i.e., substituting $\infty$ by $m$), the reuse distance sequence will be $(m)^{m}$ for all the arriving data items. 
Then, we can sum up the entire sequence to obtain RD as follows: \begin{flalign}  \label{eq:overhead_mapreduce1} 
   RD  =  \sum_{1}^{n} \sum_{1}^{m}m = nmm =  \mathcal{O}(nm^2). 
\end{flalign}
\end{proof}

Next, we can estimate the benefit of mini-batching (as described in Algorithm \ref{alg:batch}) for reducing the reuse distance. 

\begin{theorem}
Mini-batching can reduce the reuse distance of an ensemble implementation by a constant factor.
\end{theorem}

\begin{proof}

Consider an ensemble implementation like Algorithm \ref{alg:batch}, whose computational cost is dominated by the   training phase. With mini-batching, the access sequence will change  from $(e_1,e_2,\dots,e_m)^n$ 
(in Algorithm \ref{alg:highlevel-new}) to 
$({e_1}^b,{e_2}^b,\dots,{e_m}^b)^{n/b}$ 
(in Algorithm \ref{alg:batch}), where $b$ is the mini-batch size.
For each mini-batch, the RD sequence is $(m,1^{b-1})^m$. 
Finally, the RD sequence for the whole stream will be $((m,1^{b-1})^m)^{n/b}$, and the RD can be computed as:
\begin{flalign}  
    \label{proof:batch}
   RD  =  \sum_{1}^{\nicefrac{n}{b}} \sum_{1}^{m} (m+b-1)  = 
   \mathcal{O}(\lceil\frac{nm^2}{b}\rceil).
\end{flalign}
{Hence, the mini-batch can reduce the reuse distance by a constant  of $b$, where $b$ is the mini-batch size.}
\end{proof}

Although our demonstration assumes sequential processing, the result is also valid for the parallel execution proposed in Algorithm \ref{alg:batch}.
Notice that the outer loop in line 11 assigns a different ensemble learner for each processing core, while the innermost loop iterates over the mini-batch data items.
This loop arrangement enables the ensemble to process mini-batches of different sizes {than the one passed as a parameter}.
However, this case can only happen if the stream is terminated with an incomplete mini-batch.
In this case, the mini-batch has to be processed with {fewer} instances than expected.
Thus, each processing core needs to load only one learner model in its memory caches to process the entire mini-batch. 

\begin{table*}[ht]
\centering
\label{tab:rd_parallel}
\caption{Parallel execution of a stream of $n=6$ data items being processed by an ensemble of $m=4$ learners in 3 processors with mini-batch size $b=3$. } 
\label{tab:cache}
\begin{tabular}{cc p{0cm}p{0cm}p{0cm}p{0cm} p{0cm}p{0cm}p{0cm}p{0cm}  p{0cm}p{0cm}p{0cm}p{0cm} p{0cm}p{0cm}p{0cm}p{0cm} p{0cm}p{0cm}p{0cm}p{0cm} p{0cm}p{0cm}p{0cm}p{0cm}}

&\multicolumn{21}{c}{Reuse distance with mini-batching of size $b=3$} \\
\hline 
\multirow{2}{4em}{P1}& 
Access seq. &
$e_1,$&$e_1,$&$e_1,$&
$e_1,$&$e_1,$&$e_1$ &
$e_4,$&$e_4,$&$e_4,$&
$e_4,$&$e_4,$&$e_4,$& \dots
\\
&RD seq.  & 4,&1,&1,&1,&1,&1,&4,&1,&1,&1,&1,&1,& \dots\\
\hline
\multirow{2}{4em}{P2}& 
Access seq. &
$e_2,$&$e_2,$&$e_2,$&
$e_2,$&$e_2,$&$e_2,$& \dots
\\
&RD seq.  &4,&1,&1,&1,&1,&1,& \dots \\

\hline
\multirow{2}{4em}{P3}& 
Access seq. &
$e_3,$&$e_3,$&$e_3,$&
$e_3,$&$e_3,$&$e_3,$& \dots
\\
&RD seq.  &m,&1,&1,&1,&1,&1,&\dots \\
\hline
\end{tabular}
\end{table*}

It is worth noting such results hold regardless of the locality measure used for the demonstration. 
As formally demonstrated by YUAN et al. \cite{yuan2019relational},
locality access measures such as \textit{address independent (AI) sequence},  \textit{reuse interval (RI) sequence}, \textit{per datum sequence of reuse interval (PD-RI)}, and \textit{per datum reuse distance (PD-RD)} are equivalent to each other, and they can be inter-converted.
Using these results, other measures could be seamlessly used to demonstrate that mini-batching improves access locality of the implementation of ensembles for stream processing.
We chose the measure \textit{reuse distance} because of its close relation to cache misses, as the larger the reuse distance, the higher the cache misses.
Cache misses occur when the reuse distance is big enough to fill the cache memory. 

\begin{theorem}
Mini-batching provides optimal access locality for the implementation of ensembles.
\end{theorem}

\begin{proof}
The proof is straightforward. 
With mini-batching, at least one mini-batch (of length $n$) is needed to contain all the stream elements. 
{For each ensemble learner $l_i$, processing the first data element $e_1$ in the mini-batch will result in a reuse distance of $m$ because the $l_i$'s data structures are being touched for the very first time. For all the remaining $b-1$ data elements of the mini-batch, the reuse distance will be 1 as it reuses the same datum.} Thus, every learner produces a reuse distance of $m+b-1$. This is a lower bound on the reuse distance as no other order can reduce it.  
For an ensemble of $m$ learners, the total reuse distance will be $m*(m+b-1) = \mathcal{O}(m^2)$ as $b$ is a constant. This is equal to Eq. \ref{proof:batch} when the batch size $b$ is equal to the stream length $n$.

\end{proof}

{Although using only one mini-batch to process the entire data stream is useful for demonstrating that the access locality is optimal, it is not useful in practice.
Notice that the pure stream processing (as in Algorithm \ref{alg:highlevel-new}) performs the classification and training steps for every stream's incoming data item.}
Thus, the learner models continuously evolve, and the processing of every data instance can influence the next incoming data classification.
On the other hand, with mini-batching (as proposed in Algorithm \ref{alg:batch}), the ensemble training is deferred to the end of each mini-batch.
So, setting a mini-batch size of 1 boils down to pure stream processing. In contrast, mini-batches of the same length as the entire data stream {turn} it into a pure batching scheme in which all data instances are classified using models built during an offline training phase that precedes the entire stream.  
In summary, the choice of the mini-batch size raises a trade-off between learning capabilities (with short mini-batches) and computational performance (with larger mini-batches).

Yuan et al. \cite{yuan2019relational} demonstrated that several locality measures are equivalent and may be inter-converted.
However, not all measures can be equally usable for our purpose.
We chose the \textit{reuse distance} because the demonstration that mini-batching leads to optimal locality becomes straightforward with this measure, but not using the footprint or miss ratio, for instance.

\section{Experimental evaluation}
\label{sec:experiments}

To demonstrate the impact of mini-batching for bagging ensembles, we implement the strategy in the Massive Online Analysis (MOA) framework~\cite{bifet2010moa},
and tested its performance on the six ensemble algorithms described in section 
\ref{sec:ensembles}.

To assess the impact of mini-batching, we tested five different configurations of the algorithm:  a  sequential implementation (baseline), a parallel implementation without mini-batching,  and a parallel implementation with mini-batches of varying sizes [50, 500, and 2000]. {All the parallel implementations executed with 8 threads pined to processing cores for better cache usage}.  
{We chose ensembles with 100 and 150 classifiers  similar to other works 
\cite{onl-rf}}. Also, Panda et al.\cite{planet}  demonstrated a reduction in deviance asymptotes with more than 100 classifiers.




\begin{figure*}[!ht]
    \centering
    \advance\leftskip-3cm
    \includegraphics[width=1.5\linewidth]{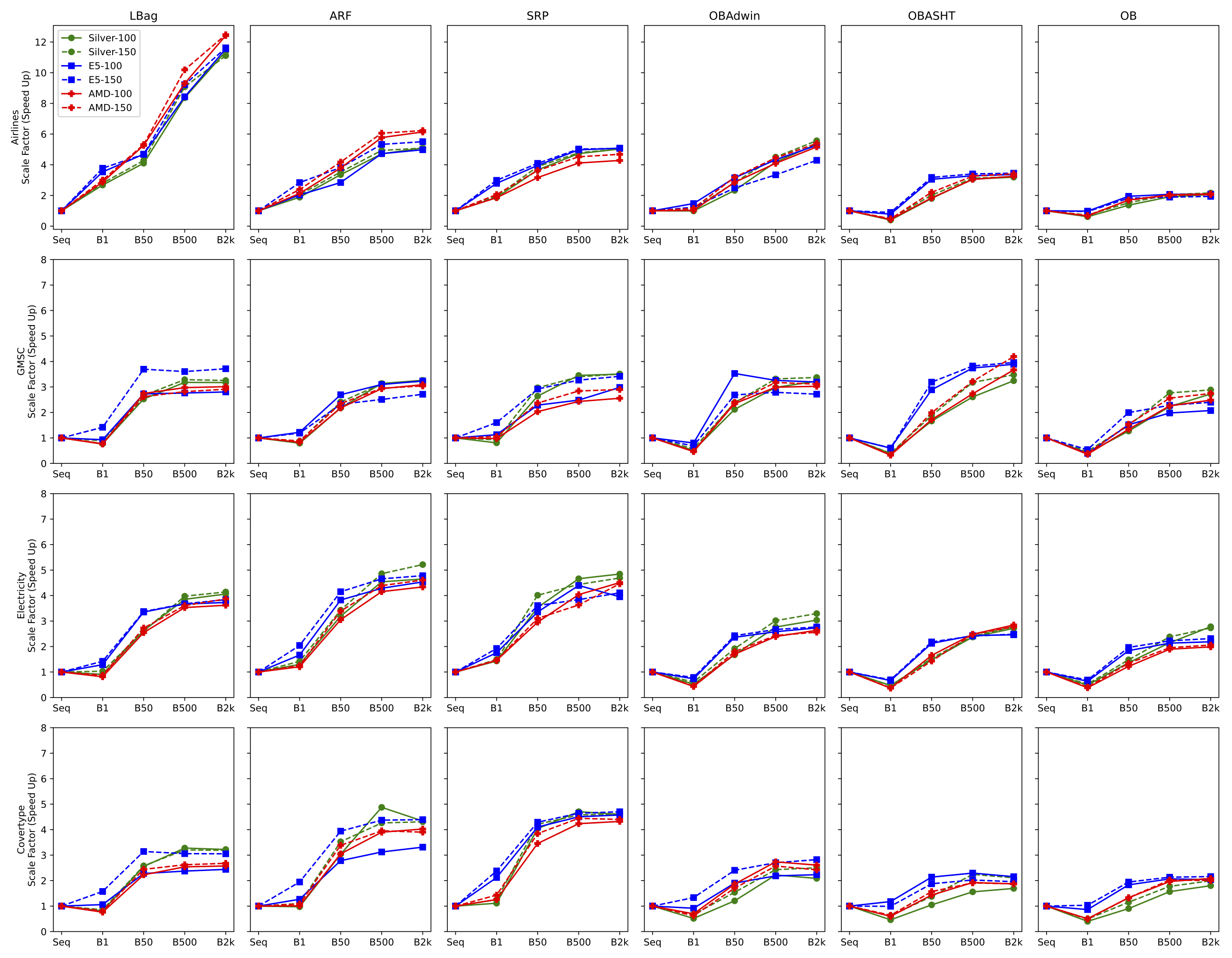}
    \caption{{\bf Speed up on all platforms}. Algorithms are placed on columns, while datasets are placed in different rows of the grid. Suffix 100 and 150 indicate the size of the ensemble, and are represented with solid and dashed lines, respectively. First row Y is scaled to 12, every other row is scaled to 8. Algorithm implementations: ($i$) Baseline (sequential)  (\textbf{Seq}), ($ii$) Parallel  (\textbf{B1}), ($iii$)  Parallel with mini-batches of 50 instances (\textbf{B50}), ($iv$)  Parallel with mini-batches of 500 instances (\textbf{B500}), ($v$)  Parallel with mini-batches of 2,000 instances (\textbf{B2k}).
    }
    \label{fig:scale_xeon}
\end{figure*}

Figure \ref{fig:scale_xeon} presents the speedup achieved in each configuration with 8 cores in each platform.
Results confirm that pure parallelism (without mini-batching) yields bad performance in many configurations. In contrast, performance gains are obtained by combining parallelism with mini-batching due to better memory access patterns. 
Also, speedups are closely related to the models' computational complexity, which varies according to the algorithm and dataset used. 
Ideally, the speedups should approach 8, which is the number of cores used. 
Cheaper algorithms (e.g., OzaBag and OzaBagAdwin) {show lower Speedups due to smaller work per thread}. However, LeveragingBag presented a Speedup of 12X for the Airlines dataset, indicating twofold benefit, i.e., gains due to the parallelism and better memory locality.
The averaged execution times are shown in Tables \ref{tab:exec_time_Silver}, \ref{tab:exec_time_E5} and \ref{tab:exec_time_AMD}.

\begin{table}[ht]
\advance\leftskip-4cm
\begin{tiny}
\caption{Silver}
\begin{tabular}{ll|rrrrrrrrrrrr}
dataset & size & LBag100 & LBag150 & ARF100 & ARF150 & SRP100 & SRP150 & OBAdW100 & OBAdW150 & OBASHT100 & OBASHT150 & OB100 & OB150 \\
\hline
Airlines & B1 & 1214.26 & 1611.74 & 1595.87 & 2228.68 & 1745.86 & 2485.18 & 1040.10 & 1337.43 & 1165.38 & 1499.49 & 512.10 & 699.88 \\
Airlines & B50 & 793.10 & 1056.62 & 895.67 & 1294.31 & 891.80 & 1255.72 & 442.25 & 530.05 & 256.12 & 356.82 & 230.52 & 330.61 \\
Airlines & B500 & 388.75 & 497.99 & 637.78 & 925.47 & 691.25 & 1017.78 & 246.44 & 335.44 & 150.64 & 230.47 & 164.73 & 245.19 \\
Airlines & B2k & 285.79 & 406.85 & 589.97 & 897.89 & 650.19 & 968.36 & 192.63 & 271.18 & 144.85 & 223.60 & 152.63 & 232.81 \\
GMSC & B1 & 176.89 & 247.17 & 222.44 & 311.86 & 385.88 & 526.25 & 95.80 & 141.18 & 80.63 & 111.66 & 100.19 & 143.77 \\
GMSC & B50 & 54.10 & 82.42 & 80.81 & 112.66 & 117.85 & 164.17 & 23.20 & 34.23 & 18.78 & 23.91 & 32.46 & 45.84 \\
GMSC & B500 & 43.01 & 66.92 & 56.08 & 87.09 & 90.20 & 143.46 & 16.49 & 24.62 & 11.97 & 14.13 & 18.46 & 24.66 \\
GMSC & B2k & 43.10 & 67.57 & 53.99 & 84.36 & 88.93 & 138.85 & 15.23 & 24.20 & 9.62 & 12.92 & 15.13 & 23.67 \\
Electricity & B1 & 52.36 & 68.86 & 63.83 & 87.28 & 112.70 & 156.21 & 37.37 & 49.67 & 38.41 & 54.18 & 31.81 & 43.69 \\
Electricity & B50 & 18.08 & 27.30 & 26.02 & 36.09 & 45.68 & 57.97 & 11.71 & 15.68 & 12.26 & 16.32 & 11.28 & 15.53 \\
Electricity & B500 & 12.32 & 17.87 & 18.28 & 25.47 & 34.63 & 52.46 & 7.09 & 9.96 & 7.64 & 10.45 & 7.18 & 9.66 \\
Electricity & B2k & 11.69 & 17.15 & 17.85 & 23.73 & 33.32 & 49.61 & 6.46 & 9.12 & 6.68 & 9.03 & 5.55 & 8.42 \\
Covertype & B1 & 1865.86 & 2601.87 & 1279.79 & 1722.16 & 3652.05 & 4827.71 & 1460.05 & 1986.71 & 1464.17 & 2064.24 & 1410.17 & 1943.88 \\
Covertype & B50 & 591.90 & 863.79 & 410.03 & 537.85 & 1005.78 & 1438.25 & 624.48 & 866.30 & 645.78 & 931.18 & 626.99 & 830.55 \\
Covertype & B500 & 461.46 & 696.23 & 255.34 & 444.90 & 863.56 & 1323.39 & 339.24 & 548.96 & 434.16 & 571.08 & 361.01 & 539.03 \\
Covertype & B2k & 469.40 & 699.06 & 287.04 & 440.19 & 883.68 & 1312.20 & 361.02 & 533.25 & 399.68 & 608.26 & 314.18 & 480.07
\end{tabular}
\label{tab:exec_time_Silver}
\end{tiny}
\end{table}

\begin{table}[ht]
\advance\leftskip-4cm
\begin{tiny}
\caption{E5}
\begin{tabular}{ll|rrrrrrrrrrrr}
dataset & size & LBag100 & LBag150 & ARF100 & ARF150 & SRP100 & SRP150 & OBAdW100 & OBAdW150 & OBASHT100 & OBASHT150 & OB100 & OB150 \\
\hline
Airlines & B1 & 859.97 & 1129.59 & 1336.02 & 1424.37 & 1049.64 & 1467.89 & 648.52 & 1097.30 & 505.86 & 662.43 & 293.57 & 486.94 \\
Airlines & B50 & 648.58 & 910.61 & 953.80 & 1059.42 & 739.67 & 1070.93 & 303.19 & 542.84 & 129.36 & 186.71 & 147.72 & 254.32 \\
Airlines & B500 & 360.56 & 462.20 & 573.81 & 759.21 & 589.29 & 872.08 & 220.77 & 404.44 & 120.14 & 174.19 & 138.67 & 244.79 \\
Airlines & B2k & 263.15 & 366.86 & 544.69 & 734.93 & 575.75 & 866.17 & 181.37 & 315.21 & 117.31 & 171.78 & 137.91 & 238.80 \\
GMSC & B1 & 121.92 & 125.42 & 116.21 & 193.07 & 231.32 & 243.99 & 53.52 & 98.67 & 46.19 & 65.30 & 85.82 & 97.16 \\
GMSC & B50 & 41.55 & 48.09 & 52.68 & 98.85 & 114.11 & 133.92 & 12.15 & 25.22 & 9.59 & 12.39 & 22.26 & 26.41 \\
GMSC & B500 & 41.03 & 49.31 & 45.84 & 91.43 & 105.03 & 119.66 & 13.16 & 24.29 & 7.40 & 10.34 & 17.06 & 23.08 \\
GMSC & B2k & 40.41 & 47.85 & 43.87 & 84.56 & 87.50 & 114.41 & 13.43 & 24.94 & 7.12 & 9.98 & 16.27 & 21.95 \\
Electricity & B1 & 27.94 & 37.12 & 35.70 & 45.05 & 76.28 & 101.48 & 20.00 & 27.66 & 20.65 & 28.76 & 17.67 & 24.55 \\
Electricity & B50 & 10.72 & 15.64 & 15.55 & 22.17 & 39.93 & 54.13 & 6.22 & 8.98 & 6.45 & 9.14 & 6.10 & 8.59 \\
Electricity & B500 & 9.80 & 14.34 & 13.88 & 19.76 & 30.48 & 50.85 & 5.69 & 8.16 & 5.64 & 8.26 & 5.26 & 7.59 \\
Electricity & B2k & 9.64 & 13.74 & 13.12 & 19.25 & 33.90 & 47.55 & 5.38 & 7.87 & 5.57 & 8.03 & 5.13 & 7.30 \\
Covertype & B1 & 1191.05 & 1206.49 & 764.25 & 747.30 & 1637.76 & 2213.71 & 918.64 & 1014.34 & 679.30 & 1157.92 & 668.20 & 893.88 \\
Covertype & B50 & 553.40 & 604.36 & 348.18 & 368.99 & 846.78 & 1227.48 & 441.80 & 562.39 & 374.57 & 612.11 & 314.32 & 475.36 \\
Covertype & B500 & 531.14 & 621.30 & 310.16 & 332.57 & 770.95 & 1134.29 & 384.59 & 500.47 & 348.92 & 567.41 & 278.88 & 432.87 \\
Covertype & B2k & 516.62 & 622.40 & 292.50 & 331.07 & 759.12 & 1120.31 & 376.49 & 479.20 & 371.70 & 586.98 & 282.49 & 428.01
\end{tabular}
\label{tab:exec_time_E5}
\end{tiny}
\end{table}

\begin{table}[ht]
\advance\leftskip-4cm
\begin{tiny}
\caption{AMD}
\begin{tabular}{ll|rrrrrrrrrrrr}
dataset & size & LBag100 & LBag150 & ARF100 & ARF150 & SRP100 & SRP150 & OBAdW100 & OBAdW150 & OBASHT100 & OBASHT150 & OB100 & OB150 \\
\hline 
Airlines & B1 & 965.96 & 1282.34 & 1115.80 & 1528.65 & 1340.47 & 1822.71 & 798.78 & 1059.23 & 872.98 & 1201.77 & 399.64 & 582.92 \\
Airlines & B50 & 526.64 & 724.42 & 626.87 & 873.30 & 785.44 & 1035.46 & 288.92 & 385.94 & 206.76 & 262.11 & 155.92 & 239.81 \\
Airlines & B500 & 297.95 & 377.98 & 411.64 & 602.36 & 603.82 & 830.28 & 201.53 & 276.82 & 124.73 & 177.67 & 130.36 & 203.03 \\
Airlines & B2k & 223.12 & 308.73 & 387.13 & 585.62 & 580.37 & 801.66 & 159.71 & 228.15 & 116.96 & 169.64 & 126.63 & 196.99 \\
GMSC & B1 & 117.02 & 175.56 & 131.92 & 196.91 & 231.88 & 329.48 & 67.87 & 105.91 & 68.12 & 99.15 & 68.46 & 105.50 \\
GMSC & B50 & 32.39 & 52.07 & 50.24 & 73.74 & 111.16 & 146.74 & 13.41 & 21.10 & 12.85 & 16.76 & 18.27 & 26.78 \\
GMSC & B500 & 29.88 & 48.40 & 37.25 & 58.68 & 92.87 & 122.26 & 10.52 & 15.90 & 8.07 & 10.38 & 10.84 & 16.07 \\
GMSC & B2k & 29.51 & 46.80 & 35.39 & 56.95 & 88.39 & 119.77 & 10.40 & 16.25 & 6.01 & 7.95 & 9.86 & 15.00 \\
Electricity & B1 & 34.47 & 47.37 & 40.56 & 57.83 & 78.01 & 112.49 & 25.50 & 35.70 & 30.20 & 42.84 & 20.97 & 29.53 \\
Electricity & B50 & 11.02 & 15.19 & 16.03 & 21.40 & 38.40 & 53.04 & 6.50 & 8.97 & 6.70 & 11.18 & 6.57 & 8.99 \\
Electricity & B500 & 7.95 & 11.48 & 11.78 & 16.52 & 28.11 & 45.35 & 4.62 & 6.63 & 4.46 & 6.46 & 4.24 & 6.18 \\
Electricity & B2k & 7.75 & 10.70 & 11.28 & 15.74 & 25.17 & 36.90 & 4.19 & 6.30 & 3.89 & 5.77 & 4.04 & 5.86 \\
Covertype & B1 & 1147.01 & 1661.16 & 757.98 & 1075.69 & 2230.79 & 3027.27 & 881.92 & 1380.66 & 871.84 & 1253.06 & 849.56 & 1269.25 \\
Covertype & B50 & 392.27 & 551.63 & 250.95 & 343.99 & 811.74 & 1126.28 & 329.01 & 502.55 & 363.89 & 511.11 & 324.02 & 470.59 \\
Covertype & B500 & 344.03 & 513.61 & 196.30 & 294.35 & 662.59 & 974.71 & 224.20 & 332.88 & 270.88 & 419.57 & 218.84 & 308.24 \\
Covertype & B2k & 339.32 & 502.85 & 190.23 & 298.38 & 648.96 & 984.36 & 235.28 & 353.36 & 278.17 & 428.52 & 206.23 & 311.84
\end{tabular}
\label{tab:exec_time_AMD}
\end{tiny}
\end{table}


As a major result, experiments show that mini-batching combined with multicore parallelism improves the performance of ensembles. In general, {introducing small mini-batches of up to 50 elements improves data reuse and reduces cache misses}. However, large mini-batches tend to increase the number of cache misses. {Such an increase happens because the larger the mini-batch}, the more heterogeneous the data instances. More heterogeneity implies that the {same classifier will access different paths} of the Hoeffding trees.

\subsection{Reuse distance analysis}
\label{sec:reuse_distance}

Next, we present experimental results that empirically demonstrate that mini-batching improves memory locality. However, we used the reuse distance histogram because it can be efficiently obtained by instrumenting the application \cite{yuan2019relational}. 
Reuse distance directly relates to cache performance, and the reuse distance histogram is a compact summary. In cache analysis, the RD histogram is related to the miss ratio of the cache. The lowest the frequency of large reuse distances, the fewer cache misses, and the better performance. 
{First, we instrumented the ensemble code to track the access order of the  ensemble learners  accessed}.
Then, we run experiments with streams of length $n = 5000$, an ensemble of $m = 100$ learners, and varied the mini-batch size $b$ as $[1, 10, 50, 100, 250]$.
The parameter $\lambda$ of the Poisson distribution used by the ensembles to process the stream elements also affects the memory access behavior. More precisely, it affects the probability of each learner of the ensemble performing the training phase.

At this point, we can group the algorithms in two categories regarding the parameter $\lambda$. The first category comprises the algorithms LBag, ARF, and SRP that use $\lambda=6$, and whose results are shown in Figure \ref{fig:RD-hist-LBag}. {We reported only one representative of the group (LBag) as the others presented similar behavior}. 
The second category includes OB, OBAdwin, and OBASHT and uses $\lambda=1$ is presented in Fig. \ref{fig:RD-hist-OzaBag}. OB was chosen as the representative, as the others behaved likewise.

\begin{figure}[ht]
   \centering
     \includegraphics[width=0.9\columnwidth]{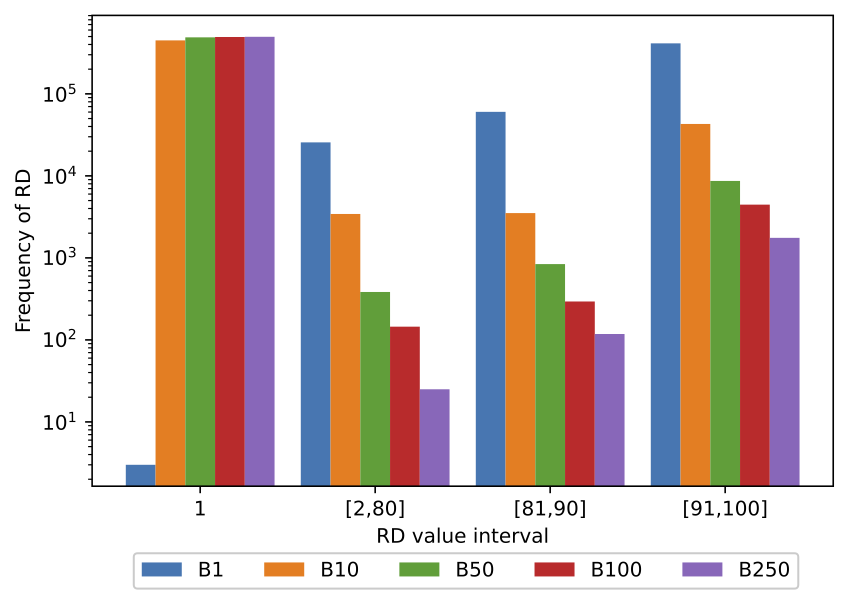} 
    \caption{RD histogram for LeveragingBag.}
    \label{fig:RD-hist-LBag}
\end{figure}
\begin{figure}[ht]
    \centering
    \includegraphics[width=0.9\columnwidth]{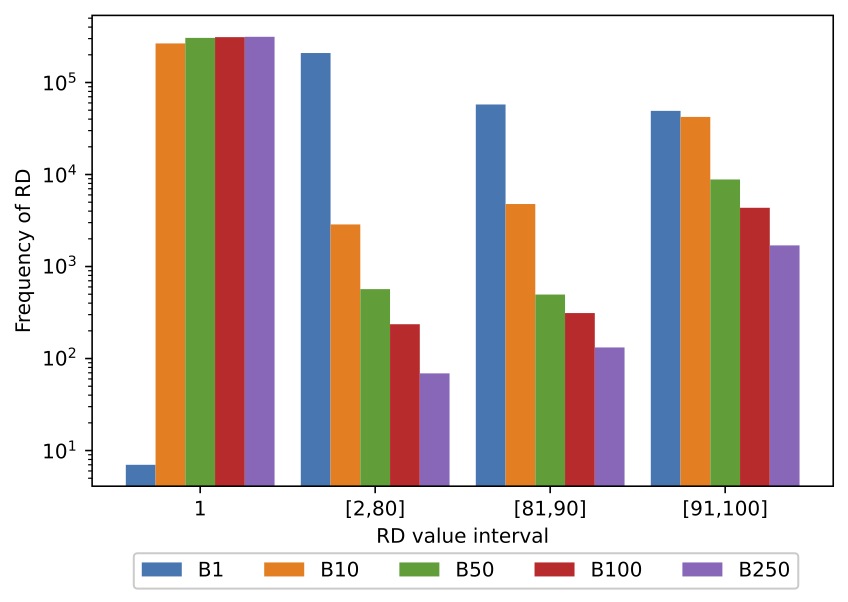} 
    \caption{RD histogram for OzaBag.}
    \label{fig:RD-hist-OzaBag}
\end{figure}

In these figures, each implementation of the same algorithm is shown by a vertical bar.
The X-axis ticks show the RD value interval in log scale, while the Y-axis shows the count of the frequencies of a given interval in the trace. 
Both figures present a very similar behavior on the mini-batch approaches. When using mini-batch, most of the RDs are 1, and the rest is in the interval [91,100]. 
Bigger mini-batch sizes present a higher RD 1 frequency and a smaller RD frequency in the interval [91,100].
The difference between the smaller and bigger mini-batch sizes is negligible compared to the changes resulting from mini-batches and the Sequential approach.

{Figure \ref{fig:RD-hist-LBag} shows the histogram for the parallel implementation (without mini-batching) and the implementations with mini-batches of different sizes (B10 to B250). Notice that the larger the mini-batch size, the less frequent the high reuse distances. For instance, near 83\% of the RD for the parallel implementation falls in the range [91,100].}
With mini-batches of size 10 (B10), high RDs' frequency decreases to near 8.6\%. With the largest mini-batch size (n=250), only 0.35\% of the RD fall in the range of [91,100]. Thus, we can conclude from this experiment that the larger the mini-batch, the fewer cache misses and the better performance.

\subsection{Memory footprint}
\label{subsec:memory}

Footprint shows the total amount of memory needed to execute the program. Next, we analyze memory footprint and cache use during the execution of ensembles of 100 learners.

\begin{table}[ht]
    \centering
    \caption{Memory footprint (in GB) for each pair (dataset, algorithm). Algorithms are ensembles with 100 learners. \vspace{0,5cm}}
    \label{tab:footprint}
    \begin{tabular}{r|cccccc|c}
         
          & \multicolumn{6}{c|}{Total footprint per algorithm}& Instance \\
            
         Dataset & ARF & SRP & LBag & OBAdwin & OBASHT & OB & size\\
         \hline
         Airlines & 20.11 & 19.31 & 20.93 & 19.24 & 14.67 & 18.42 & $4.71 \times10^{-5}$\\
         GMSC & 13.15 & 13.80 & 9.15 & 3.05 & 1.90 & 2.99 & $0.17 \times10^{-5}$\\
         Electricity & 9.86 & 14.78 & 7.75 & 1.47 & 1.03 & 1.15 & $0.22 \times10^{-5}$\\
         Covertype & 16.15 & 17.69 & 16.84 & 2.94 & 2.45 & 3.13 & $2.15 \times10^{-5}$\\
    \end{tabular}
\end{table}
Table \ref{tab:footprint} shows the maximum memory footprint for each pair (dataset, algorithm). The rightmost column shows the instance size for each dataset expressed in GB.   
Values represent the average of three executions.
The Airlines dataset {has the largest memory footprint} on all algorithms.
In special, LBag has {the largest footprint for the Airlines} dataset.
The combination of LBag and airlines dataset is the one that presents both the largest memory footprint and the largest Speedup, suggesting that its superlinear Speedup is more closely related to the improved memory usage and data locality. The mini-batching does not modify the footprint, as it only changes the ordering of access to the data structures, while the number of accesses remains unchanged.

\subsection{Cache misses}

Table \ref{tab:perf-cache} shows two measures counters related to cache use provided by the perf tool\footnote{\url{https://man7.org/linux/man-pages/man1/perf.1.html}}:
\begin{itemize}
    \item {\it Cache-references}: accounts for data requests missed in the L1 and L2 caches. Whether they miss the L3 is irrelevant in this case; 
    \item {\it Cache-misses}: this event represents the number of memory access that could not be served by any of the cache levels, therefore having to fetch data from the main memory.
\end{itemize}
The difference between the two is the number of {L3 hits}.
Although the two measures change according to the dataset size, both tend to decrease with mini-batching and larger batch sizes. 
For the  Electricity and Covertype case, the {\it cache\_refer} starts to rise with mini-batches of 2000 instances, suggesting the existence of an optimal mini-batch size. 
Results show clearly that mini-batching improves the memory accesses locality, thus reducing cache misses.
Originally, the data structures of each learner were accessed once (then discarded) for each data instance of the mini-batch.

\begin{table}[ht]

\advance\leftskip-1,0cm
\begin{tiny}
\caption{Measures of cache usage for ensembles with 100 learners (Millions)}
\label{tab:perf-cache}
\begin{tabular}{c|r|p{1,0cm}p{1,0cm}|p{1,0cm}p{1,0cm}|p{1,0cm}p{1,0cm}|p{1,0cm}p{1,0cm}}
&  & \multicolumn{2}{c}{Airlines} & \multicolumn{2}{c}{GMSC} & \multicolumn{2}{c}{Electricity} & \multicolumn{2}{c}{Covertype} \\
Algorithm & MB size & cache-miss & cache-refer & cache-miss & cache-refer & cache-miss & cache-refer & cache-miss & cache-refer \\
\hline
ARF & 1 & 40,171 & 94,910 & 2,518 & 11,366 & 882 & 4,490 & 12,652 & 65,321 \\
ARF & 50 & 41,634 & 63,303 & 2,499 & 4,825 & 821 & 2,323 & 13,325 & 23,201 \\
ARF & 500 & 42,321 & 62,185 & 2,162 & 4,394 & 742 & 2,040 & 12,315 & 21,246 \\
ARF & 2000 & 42,522 & 61,728 & 2,047 & 4,369 & 728 & 2,293 & 12,367 & 21,912 \\ \hline
LBag & 1 & 45,337 & 99,010 & 2,600 & 8,962 & 508 & 2,870 & 17,809 & 104,735 \\
LBag & 50 & 49,425 & 74,854 & 1,680 & 3,746 & 516 & 1,706 & 16,560 & 47,024 \\
LBag & 500 & 26,659 & 37,783 & 1,645 & 3,497 & 473 & 1,457 & 19,315 & 45,322 \\
LBag & 2000 & 19,556 & 26,152 & 1,546 & 3,714 & 463 & 1,309 & 21,342 & 48,600 \\ \hline
SRP & 1 & 45,135 & 110,900 & 5,543 & 18,487 & 2,105 & 7,520 & 65,157 & 172,089 \\
SRP & 50 & 46,647 & 68,340 & 5,285 & 8,682 & 1,999 & 3,892 & 61,763 & 97,867 \\
SRP & 500 & 45,973 & 67,255 & 4,781 & 7,750 & 1,952 & 4,296 & 60,210 & 95,699 \\
SRP & 2000 & 45,973 & 66,395 & 4,559 & 7,912 & 1,863 & 3,916 & 60,906 & 99,117 \\ \hline
OBASHT & 1 & 4,779 & 39,986 & 531 & 4,399 & 225 & 1,714 & 5,927 & 101,370 \\
OBASHT & 50 & 3,918 & 10,629 & 399 & 1,262 & 171 & 781 & 5,286 & 40,059 \\
OBASHT & 500 & 3,810 & 9,953 & 353 & 1,033 & 157 & 717 & 4,648 & 36,992 \\
OBASHT & 2000 & 3,579 & 9,603 & 334 & 1,090 & 155 & 761 & 4,302 & 39,074 \\ \hline
OBAdwin & 1 & 26,627 & 71,987 & 723 & 5,770 & 232 & 2,037 & 5,780 & 108,281 \\
OBAdwin & 50 & 20,338 & 30,542 & 439 & 1,539 & 183 & 910 & 4,687 & 37,948 \\
OBAdwin & 500 & 15,417 & 21,888 & 419 & 1,357 & 177 & 872 & 5,576 & 35,341 \\
OBAdwin & 2000 & 11,669 & 16,414 & 371 & 1,427 & 149 & 915 & 6,228 & 33,759 \\ \hline
OB & 1 & 9,423 & 27,560 & 981 & 5,580 & 221 & 1,864 & 11,314 & 94,976 \\
OB & 50 & 9,810 & 13,606 & 635 & 1,853 & 180 & 735 & 9,683 & 36,822 \\
OB & 500 & 9,504 & 12,468 & 421 & 1,531 & 173 & 738 & 7,983 & 32,385 \\
OB & 2000 & 8,965 & 12,299 & 353 & 1,386 & 155 & 793 & 7,141 & 32,146
\end{tabular}
\end{tiny}
\end{table}

To conclude, the results presented in this section demonstrate that
mini-batching can significantly improve the performance of parallel implementations of 
ensembles on multicore systems. 
{We used three different measures to demonstrate the benefits of mini-batching on memory access}: ($i$) reduction in the reuse distances demonstrated in section \ref{sec:background}; ($ii$) reduction of the frequency of large reuse distance demonstrated by the reuse distance histograms in Section \ref{sec:reuse_distance}, and ($iii$) reduction of cache misses in the L1 and L2, as demonstrated by the {\it cache-refer} event counter. All these demonstrations using different measures (whose equivalence was proved in \cite{yuan2019relational})  corroborate each other to show the benefits of mini-batching.

{Experiments in section 7.1} evaluate the influence of the  increase of the mini-batch size on the speedup. Results show that the increase of the mini-batch size up to 50 instances yields a significant increase in speedups, while the increase of the mini-batch beyond this size leads to slight speedup increases. 
Similarly, experiments in section 7.2 demonstrate that mini-batches of up to 50 instances still {yield} a significant reduction {in RD. In comparison, mini-batches} larger than 50 instances presented a slight reduction in RD. 
{Likewise, Table 11 (in section 7.4)  shows similar results regarding the mini-batch size, using the cache-references measure}. 

As demonstrated in our experiments, the mini-batching can alleviate the memory bottleneck and increase the speedup to some extent. Any increase {in processing cores} beyond this threshold would hardly increase the speedup because of the memory bottleneck. {To achieve} even higher scalability, further strategies are necessary. {For instance, a promising research direction is to design algorithms that work under more strict memory constraints}, whose data structures can live in the cache memories.

Regarding general guidelines of a standard setting of the mini-batch size, the empirical results showed that the increase {in the mini-batch size yields significant} performance improvements for mini-batches up to 50 instances. After this threshold, experiments show diminishing returns. However, the increase of mini-batch size can impact the predictive performance. {We investigate this impact  in the next section}. 

\section{Impact of mini-batching on predictive performance } 
\label{sec:batch-size}

Streaming ensemble algorithms operate in a sample-wise mode, i.e., the classifiers are up-to-date any time after each sample and can immediately predict the next. In contrast, with the introduction of mini-batching, 
{first, the} classifiers output predictions for all the instances of the new mini-batch and then  train (i.e., update the models) to classify the next mini-batch, and so on. 
So, what is the cost of deferring the training  on the predictive classification? Furthermore, does the mini-batch size impacts the predictive performance? Next, we empirically address these questions.

\subsection{The impact of deferred updates}

To evaluate the impact of mini-batching (and its deferred training) on the predictive performance, we must ensure that the proposed algorithm is as close as possible to the baseline.
One aspect that can produce {a} discrepancy in predictive performance {is using} a different sequence of pseudo-random numbers to compute weights in the  Poisson distribution.
To guarantee the same pseudo-random sequence, we used the same seed and assigned the weights in the same order (i.e., sequentially).
Thus, the only difference between the baseline and our mini-batching implementation is the deferred training. {We used seven mini-batch sizes (25, 50, 100, 250, 500, 1000, and 2000)}. We compare the precision and recall of the mini-batching and the baseline implementations.
Figure \ref{fig:accxsize} shows the precision and recall for each combination of dataset and algorithm.

\begin{figure}[hbt]
    \centering
    \includegraphics[width=9cm, height=11cm]{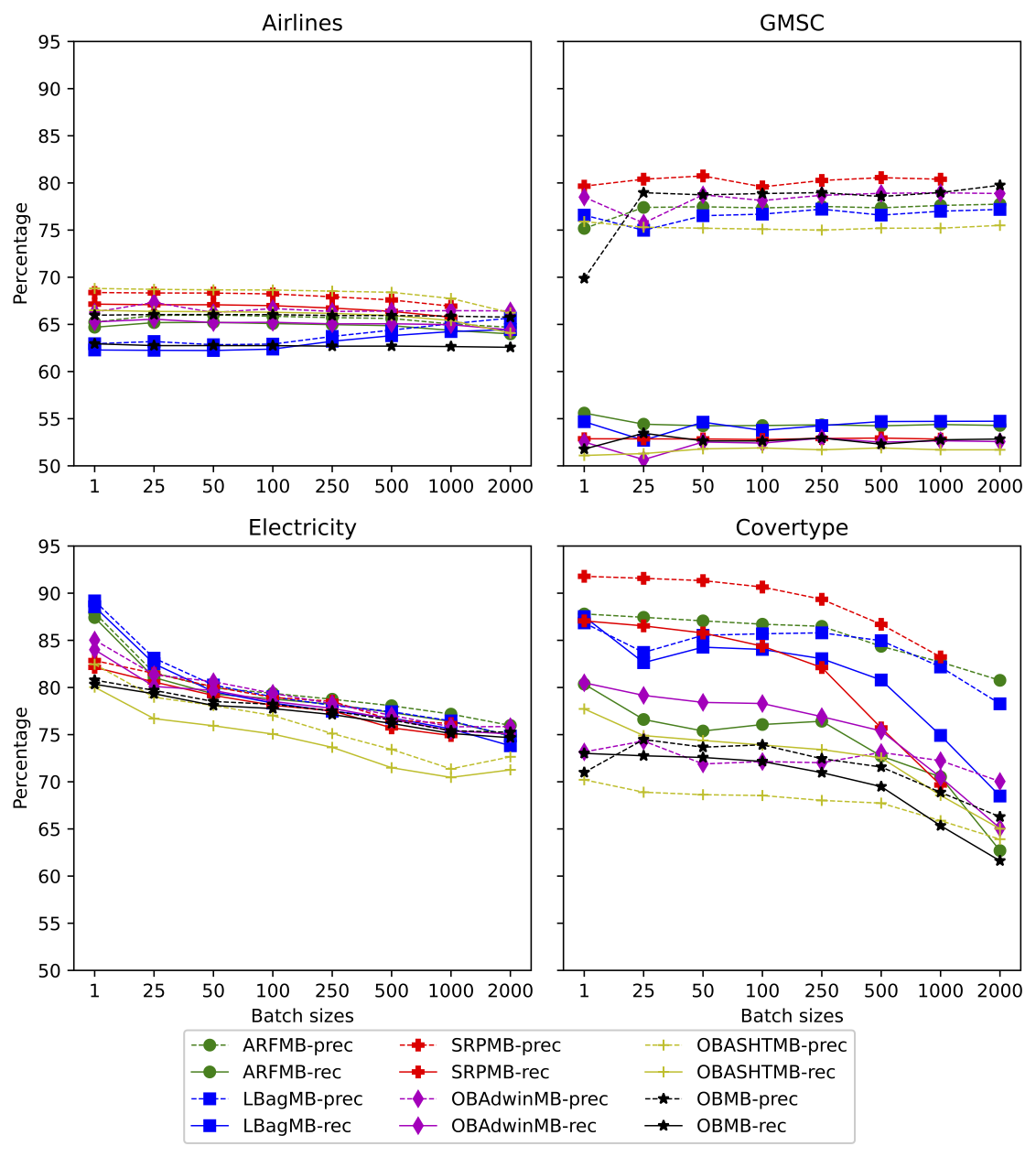}
    \caption{
    The {\it precision} and {\it recall} measures  for each algorithm grouped by datasets.
    The {Y-axis} shows predictive performance as {a} percentage value. The {X-axis} shows the mini-batch size. 
    Solid lines are used for {\it recall} and dashed lines for {\it precision}. 
    }
    \label{fig:accxsize}
\end{figure}

We can observe two distinct behaviors. The increase in the mini-batch size has {a} low impact on predictive performance for the datasets Airlines and GMSC. In contrast, a more significant decrease in predictive performance occurs with Covertype and Electricity. 
Results show that the impact on the predictive performance is more influenced by the dataset characteristics than by the algorithms. 

\subsection{The impact on change detection}

Additional experiments using LBag and OBAdwin were carried out to track the number of changes detected in each dataset. As a general remark, the number of changes detected decreases as the mini-batch size increases. Small mini-batches ({i.e.,} less than 50 instances) deviate from this behavior, as the algorithms detect more changes than baseline.

The spike in the number of changes is caused by the lower accuracy, which is related to the slower pace that the models are trained.
The behavior described can be viewed in Figure \ref{fig:CD-lines} and they can help explain some prediction results from Figure \ref{fig:accxsize}.

\begin{figure}[!hbt]
    \centering
    \includegraphics[width=0.65\columnwidth]{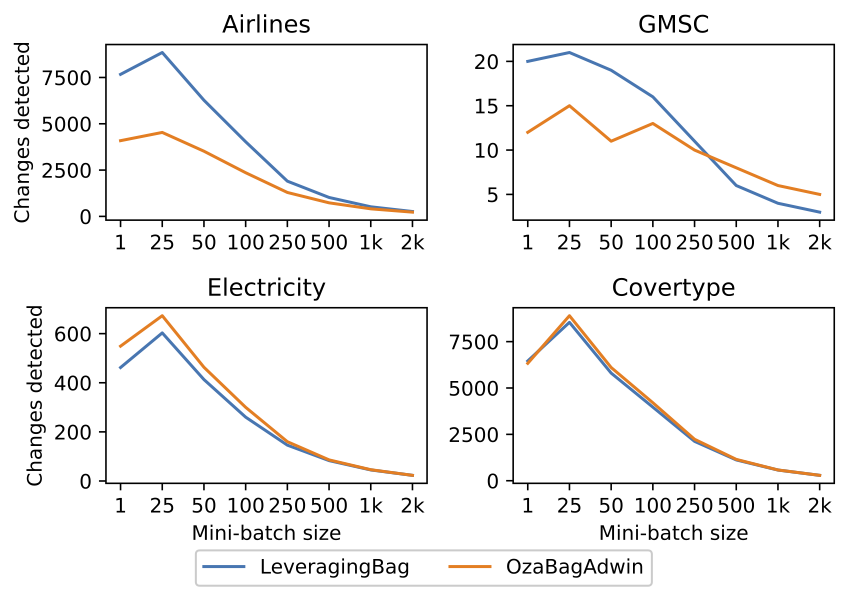} 
    \caption{Changes detected by algorithms LeveragingBag (LBag) and OzaBagAdwin (OBA) for each dataset according to the mini-batch size.}
    \label{fig:CD-lines}
\end{figure}

{The GMSC dataset} shows a minimal amount of changes detected, indicating that data distribution is stable and consistent. The consequence is that the ensemble models are rarely replaced, as only 20\% of the ensemble is replaced over the full length of the 150,000 instances. On the other hand, the Electricity dataset presents over 600 changes detected for the same ensemble size (100) {and only a third (45,312) of the instances, which means the ensemble replaced each model 6 times. Such replacement} indicates that the Electricity dataset has many data distribution changes. Therefore models become obsolete quicker and need to be replaced. This behavior explains why GMSC has a very stable predictive performance {while Electricity's predictive performance deteriorates} on all algorithms as the mini-batch size increases.

The number of changes detected is similar in {the Airlines} and {the Covertype} datasets. However, {the predictive performance is constant in the Airlines dataset while it deteriorates in the Covertype dataset}. {One difference is that} the Airlines dataset is a binary class dataset while the Covertype dataset is multiclass. Another difference is that the Airlines dataset has two nominal attributes with many values, {which tricks the classifier into performing splits in these attributes,} as shown in~\cite{feat-score}. The mini-batch impact in prediction is minimal in the Airlines dataset because the initial prediction is already low. On the other hand, {the impact on Covertype dataset} is more noticeable because the original prediction (in the incremental setting) is better and deteriorates with mini-batches.

Figure \ref{fig:CD-lines} shows an increase in changes detected for small batches up to 25 {instances}. This behavior happens in the beginning when the models are untrained and thus have a high error rate, triggering change detections more often. As the stream progresses, the models grow and become more capable of recognizing the classes, thus improving this behavior. 
In larger mini-batches, there are significantly {fewer} opportunities to trigger change detection, thus decreasing the predictive performance.

The experiments presented so far revealed a trade-off between speedup and predictive performance. Experiments presented in Section 6 show that larger speedup increases come up to mini-batches of {50 instances. In comparison, experiments} in this Section show a significant loss in predictive performance for mini-batches larger than 100 instances. 
Thus, a general guideline for a good compromise between speedup and predictive performance is to use mini-batches of 50 to 100 instances. 
The optimal choice for the mini-batch size can change depending on the {data characteristics} (the tuning of this parameter for specific datasets is out of this work's scope). However, as we observe from experiments shown in Section 8.1, the impact on  predictive performance has low sensitivity changes in the mini-batch size for most datasets. 
Thus, {for situations with no previous knowledge} about data characteristics, using mini-batches of 50 instances can  be a conservative initial setup for the algorithms. 

\section{Conclusion}
\label{sec:conclusions}

Ensemble learning is a fruitful approach to improve the performance of ML models by combining several single models. Examples of this class include the algorithms Adaptive Random Forest, Leveraging Bag, and OzaBag. This approach is also popular in a data stream processing context. Despite their importance, many aspects of their efficient implementation remain to be studied. 
This paper highlighted a performance bottleneck of multi-core implementations of bagging ensembles and proposed a mini-batching strategy to improve the memory access pattern's locality and reduce processing time. We demonstrated through theoretical and experimental frameworks that the performance achieved by multicore parallelism {could} be remarkably improved (speedups of up to 5X on 8-core processors) through the application of this mini-batching technique.
We observed that the choice of the mini-batch size {could} raise a trade-off between computational performance and predictive performance. 
{Our experiments with six bagging ensemble methods  and four datasets showed a good compromise solution  around mini-batches of 50 instances}.
However, the choice of {the} best mini-batch sizes may depend on the application scenario.
As a final comment, we believe that mini-batching can support manifold performance improvements {to implementations of bagging ensembles}. 

As future work, we intend to investigate how mini-batching can improve {the} energy efficiency {of bagging ensembles}. 
Furthermore, the use of varying size mini-batches, {i.e.,} which can be dynamically adjusted at run-time according to some variable parameters ({e.g.,} the incoming rate of instances, the delay {in processing each} instance in the current window, the data characteristics, the accuracy measured at the output during a time window, to cite a few) can be an interesting direction for future research.

\section*{Acknowledgements}
This study was financed in part by the Coordenação de Aperfeiçoamento de Pessoal de Nível Superior - Brasil (CAPES) - Finance Code 001, and Programa Institucional de Internacionalização – CAPES-PrInt UFSCar (Contract 88887.373234/2019-00). Authors also thank Stic AMSUD (project 20-STIC-09), and FAPESP (contract numbers  2018/22979-2, and 2015/24461-2) for their support.
Partially supported by the TAIAO project CONT-64517-SSIFDS-UOW (Time-Evolving Data Science / Artificial Intelligence for Advanced Open Environmental Science) funded by the New Zealand Ministry of Business, Innovation, and Employment (MBIE). URL https://taiao.ai/.


\bibliography{sample-base}

\end{document}